%% file: GCIR-TIFS.tex
\documentclass[lettersize,journal]{IEEEtran}
\usepackage{amsmath,amsfonts}
\usepackage{algorithmic}
\usepackage{array}
\usepackage[caption=false,font=footnotesize,labelfont=sf,textfont=sf]{subfig}
\usepackage{textcomp}
\usepackage{stfloats}
\usepackage{url}
\usepackage{verbatim}
\usepackage{graphicx}
\usepackage{cite}
\usepackage{wrapfig}
\usepackage{caption}
\usepackage{subfig} 
\usepackage{subfloat}
\usepackage{subcaption}
\usepackage{cite}
\usepackage{booktabs}
\usepackage{amsmath,amssymb,amsfonts}
\usepackage{algorithmic}
\usepackage{graphicx}
\usepackage{diagbox}
\usepackage{textcomp}
\usepackage{xcolor}
\usepackage[numbers, sort&compress]{natbib}
\usepackage{hyperref}
\usepackage{ulem}
\def\BibTeX{{\rm B\kern-.05em{\sc i\kern-.025em b}\kern-.08em
		T\kern-.1667em\lower.7ex\hbox{E}\kern-.125emX}}
	
\usepackage[ruled,vlined]{algorithm2e}
\usepackage{multirow}
\usepackage{soul}
\usepackage{bbm}
\usepackage{amsthm}
\makeatletter
\hypersetup{
    colorlinks=true,
    linkcolor=blue,
    filecolor=magenta,
    urlcolor=cyan,
}

\newtheorem{theorem}{Theorem}

\newtheorem{lemma}[theorem]{Lemma}

\newtheorem{definition}[theorem]{Definition}

\DeclareMathOperator{\Tr}{Tr}
\begin{document}

\title{Robust Graph Contrastive Learning with Information Restoration}

%\author{\IEEEauthorblockN{Yulin Zhu\IEEEauthorrefmark{1}, Tomasz Michalak\IEEEauthorrefmark{2}, Xiapu Luo\IEEEauthorrefmark{1}, Kai Zhou\IEEEauthorrefmark{1}}\\
%\IEEEauthorblockA{\IEEEauthorrefmark{1}\textit{Dept. of Computing}, \textit{The Hong Kong Polytechnic University}, HKSAR\\ 
%		yulinzhu@polyu.edu.hk, csxluo@comp.polyu.edu.hk, kaizhou@polyu.edu.hk}\\
%\IEEEauthorblockA{\IEEEauthorrefmark{2}\textit{University of Warsaw \& Ideas NCBiR}, Nowogrodzka 47A 00-695 Warszawa, Poland, tpm@mimuw.edu.pl}}
\author{Yulin Zhu, Xing Ai, Yevgeniy Vorobeychik, and Kai Zhou\thanks{Prof. Kai Zhou is the corresponding author.}
\thanks{Dr. Yulin Zhu is with the Department of Computer Science, Hong Kong Chu Hai College. E-mail: {\tt ylzhu@chuhai.edu.hk}. Mr. Xing Ai and Prof. Kai Zhou are with the Department of Computing, The Hong Kong Polytechnic University. E-mail:{\tt xing96.ai@connect.polyu.hk, kaizhou@polyu.edu.hk}. Prof. Yevgeniy Vorobeychik is with the Department of Computer Science and Engineering, Washington University in St. Louis. E-mail: {\tt yvorobeychik@wustl.edu}.}
\thanks{This work has been submitted to the IEEE for possible publication.
	Copyright may be transferred without notice, after which this version may
	no longer be accessible.} 
%\\ 
%Xiaoge Zhang is with the Department of Industrial and Systems Engineering, The Hong Kong Polytechnic University. E-mail: {\tt xiaoge.zhang@polyu.edu.hk}. 
%\\
%Tomasz Michalak is with the University of Warsaw \& Ideas NCBR, Warsaw, Poland. E-mail: {\tt tpm@mimuw.edu.pl}.}
%\thanks{Manuscript received April 19, 2021; revised August 16, 2021.}
}



\maketitle

\begin{abstract}
The graph contrastive learning (GCL) framework has gained remarkable achievements in graph representation learning. However, similar to graph neural networks (GNNs), GCL models are susceptible to graph structural attacks. As an unsupervised method, GCL faces greater challenges in defending against adversarial attacks. Furthermore, there has been limited research on enhancing the robustness of GCL. To thoroughly explore the failure of GCL on the poisoned graphs, we investigate the detrimental effects of graph structural attacks against the GCL framework. We discover that, in addition to the conventional observation that graph structural attacks tend to connect dissimilar node pairs, these attacks also diminish the mutual information between the graph and its representations from an information-theoretical perspective, which is the cornerstone of the high-quality node embeddings for GCL. Motivated by this theoretical insight, we propose a robust graph contrastive learning framework with a learnable sanitation view that endeavors to sanitize the augmented graphs by restoring the diminished mutual information caused by the structural attacks. Additionally, we design a fully unsupervised tuning strategy to tune the hyperparameters without accessing the label information, which strictly coincides with the defender's knowledge. Extensive experiments demonstrate the effectiveness and efficiency of our proposed method compared to competitive baselines.
\end{abstract}

\begin{IEEEkeywords}
Robust Graph Contrastive Learning, Graph Representation Learning, Adversarial Robustness
\end{IEEEkeywords}

\input{sections/intro}

\input{sections/related}

\input{sections/methods}

\input{sections/experiments}

\input{sections/conclusion}

\bibliographystyle{IEEEtran}
\bibliography{citation}

%\bibliographystyle{ieeetr}
%\bibliography{citation}

\end{document}

%% file: sections/intro.tex
\section{Introduction}
\label{Sec-intro}
Graph representation learning~\cite{node2vec, DeepWalk} has revolutionized the analysis of graph data, which is prevalent across diverse domains. In practice, the scarcity of ground-truth labels has led to a surge in research on unsupervised graph learning approaches. Among these methods, \underline{G}raph \underline{C}ontrastive \underline{L}earning (GCL)~\cite{GRACE, GCA, ARIEL, SPAN, 10597723} has emerged as a highly effective unsupervised approach, outperforming other methods on various downstream tasks, including semi-supervised node classification~\cite{GNN}.   

% GCL and homophily.
The label-preserving property of GCL is a key factor behind its main benefits~\cite{autogcl}. Specifically, the node embeddings generated by GCL, which incorporates both topological and semantic information, are consistent with node label information even without explicitly querying the node labels. This consistency arises from the \textit{homophily}~\cite{homophily, 10889118, GNNhomophily, HeteRobust} assumption of the graph data, which states that nodes tend to connect with ``similar" others. This phenomenon is widely observed in real-world graph data such as friendship networks~\cite{homophily}, political networks~\cite{politicalnetwork}, citation networks~\cite{CitationNetwork}, etc. By contrasting positive and negative samples of nodes that are similar or dissimilar in semantic information, the GCL framework encourages the GNN encoder to learn node embeddings that capture the homophily patterns present in the graph. 

% GCL is vulnerable to structural attacks.
However, similar to the vanilla \underline{G}raph \underline{N}eural \underline{N}etwork (GNN)~\cite{GNN, graphsage} and its variants, GCL models are also susceptible to graph structural attacks~\cite{Nettack, Mettack, CLGA, HRAT, TopologyAttack, GODZISZEWSKI2024104173, 10535517, 10693599, 10433700, 10296883, 10446170, 10646863,9679165}. 
In such attacks, the adversary can manipulate the graph topology by adding or deleting edges in the original graph to undermine the quality of node embeddings.
For instance, a malicious entity can manipulate social ties by connecting with normal accounts to evade graph-based anti-fraud systems~\cite{BinarizedAttack}. 
Indeed, existing research has already demonstrated the vulnerability of GCL under structural attacks. In particular, CLGA~\cite{CLGA} is an attack method specifically designed against GCL and has demonstrated effective attack performance.
Moreover, although Mettack~\cite{Mettack} is an attack method designed to attack semi-supervised GNN, extensive experiments~\cite{CLGA} demonstrate that it can also successfully attack GCL, at times generating attacks that are more potent than CLGA. 
Therefore, investigating potential countermeasures is essential to ensuring the security and robustness of GCL. 

% discuss defense on GCL.
Our main goal is to \textit{refine the GCL framework so that it is robust against graph structural attacks}. While there have been a series of advances in devising countermeasures against structural attacks for GNN models~\cite{RGCN, ProGNN, GCNJaccard, SimPGCN, GRV, HeteRobust, FocusedCleaner, XLX, GNNGUARD, GASOLINE}, they are primarily designed for semi-supervised settings and intimately require access to node label information. Since GCL is fully unsupervised, such methods are not readily applicable. Two recent approaches, ARIEL~\cite{ARIEL} and SPAN~\cite{SPAN}, specifically address the challenge of learning robust GCL models. ARIEL~\cite{ARIEL} employs an adversarial training approach by introducing an adversarial view into the contrastive learning process to improve robustness. SPAN~\cite{SPAN}, on the other hand, modifies the topology augmentation process by maximizing and minimizing the spectrum of the original graph to generate augmentation views.
However, 
%these methods have significant limitations. The 
the major drawback with these methods is that their strategy to mitigate graph structural attacks has the effect of further poisoning the graph's structure. For instance, the adversarial view of ARIEL injects adversarial noises into both the feature and topology spaces of the original
poisoned graph. Nevertheless, the adversarial noises could still have
persistent negative impact on graph representation learning, leading to limited defense performance. The same problem also exists in one of the augmented views of SPAN whose goal is to minimize the spectrum of the graph Laplacian. 

%To break through the above-mentioned bottleneck, I
To avoid the above problem, in this paper, we take a quite different angle to achieve adversarial robustness, by
%we endeavor to 
\textit{restoring} the graph's properties contaminated by adversarial attacks. Intuitively, refining the poisoned graph is more effective than adversarial training since the latter cannot omit the harmful effects of the existing adversarial noises. Besides, the observation that SPAN achieves better robustness than ARIEL is because of an augmented view generated by maximizing the spectrum of graph Laplacian, which can restore the diminished graph homophily~\cite{homophily} and thus lead to a ``cleaner" graph. Therefore, restoring the graph property is a promising and effective way to improve GCL robustness against attacks.
%Thus, it is vital to restore the graph's property to achieve robustness.
%Hence, our goal is to scrutinize the relationship between the vital graph's properties and graph structural attacks 

%In order to scrutinize the vital graph's properties, 
To this end, 
we begin by assessing the vulnerability of the GCL framework to graph structural attacks. Our focus is on identifying important graph properties that are significantly changed during the attack. Apart from the common observation that structural attacks will typically reduce graph homophily, we discover that these attacks will essentially degrade the mutual information estimation between the graph and its representations (refer to Sec.~\ref{sec-analyze-infoNCE}). This sheds light on how attacks affect the graph learning tasks from an information-theoretic perspective. Specifically, the mutual information between the graph and its representations serves as a crucial cornerstone for generating high-quality node representations within the GCL framework. The graph attacker undermines the quality of the GCL's node embeddings by reducing the alignment between the graph's information (including semantic and topology information) and the low-dimensional representations. Consequently, the poisoned GCL's node embeddings fail to capture the ``true" information of the graph data. This critical observation underscores the significance of restoring mutual information to effectively defend against such attacks.

Furthermore, GCL models are unsupervised approaches, meaning that the GCL's node embeddings are pre-trained without accessing the node labels. However, the current GCL models~\cite{GCA,ARIEL,SPAN,MVGRL,DGI} often rely on the 
\textit{validation label set} to tune the hyperparameters of the GCL framework, which contradicts the principles of unsupervised learning. 
%Therefore, we encounter two main challenges to obtaining reliable GCL node embeddings under the adversarial environment:
Therefore, training robust GCL models in a fully unsupervised setting in an adversarial environment encounters two major challenges: \textbf{1)} How to restore the diminished mutual information caused by the noisy data during the training of the GCL framework? \textbf{2)} How to determine the vital hyperparameters of the GCL framework in a fully unsupervised setting?

% Main idea of our approach.
To address the challenges, we propose a robust GCL framework termed \textbf{GCIR} that endeavors to restore diminished mutual information after structural attacks, thus achieving adversarial robustness. Specifically, \textbf{GCIR} novelty integrates a learnable \textit{sanitation view} into the framework and contrasts it with an augmented view (through feature masking or link dropping~\cite{GRACE,GCA}) to gradually restore the diminished mutual information.
%during training 
%and obtain higher quality node embeddings. 
The sanitation view is trained jointly with the GNN encoder 
%and thus forms an end-to-end training framework. 
in an end-to-end manner. 
%To tackle the second problem, 
To train the model in a fully unsupervised manner, we design an early stopping strategy based on a pseudo-normalized cut~\cite{NormCut} metric without requiring the label's information for the downstream tasks to determine the best choice of the vital hyperparameter.

The main contributions are summarized as follows:
\begin{itemize}
    \item We provide the vulnerability analysis of the GCL framework under graph structural attacks from an information theoretical perspective and verify that \textit{graph structural attacks will degrade the mutual information estimation between the graph and its representations}. This finding advances our understanding of attacks and provides valuable insights for developing defense approaches.
    \item We propose a novel robust GCL framework, termed \textbf{GCIR}, that achieves adversarial robustness by introducing a learnable sanitation view to restore the diminished mutual information during training. Moreover, we design an unsupervised tuning strategy to determine the best choice of the vital hyperparameter without label information. 
    \item Extensive experiments demonstrate the effectiveness of our proposed robust GCL framework across various attack scenarios.
\end{itemize}

%% file: sections/related.tex
\section{Related Works}
\label{sec-related}
\subsection{Graph Structural Attack}
Graph-based machine learning models have been shown to be vulnerable to structural attacks. 
%Nettack~\cite{Nettack} was the first targeted attack on GNNs that manipulated both on node attribute matrix and adjacency matrix. 
Mettack~\cite{Mettack} formulated the global structural poisoning attacks on GNNs as a bi-level optimization problem and leveraged a meta-learning framework to solve it. 
%BinarizedAttack~\cite{BinarizedAttack} simplified graph poisoning attacks against the graph-based anomaly detection to a one-level optimization problem and optimize it by mimicking the training of the binary neural network. 
%HRAT~\cite{HRAT} proposed a heuristic optimization model integrated with reinforcement learning to optimize the structural attacks against Android Malware Detection. 
CLGA~\cite{CLGA} deployed GCA~\cite{GCA} as the surrogate model and formulated graph structural attacks as a bi-level optimization problem to degenerate the performance of GCL models through poisoning. There are other structural attacks against graph-based machine learning models~\cite{BinarizedAttack, HRAT, TopologyAttack, bojchevski2019adversarial} such as graph anomaly detection, malware detection, random walks, etc. 

\subsection{Graph Contrastive Learning}
Contrastive learning is a widely used deep learning model that originated from SimCLR~\cite{SimCLR}, a simple self-supervised contrastive framework for learning useful representations for image data. 
%SimCLR introduces the infoNCE loss to approximate the mutual information between latent representations from different augmentation views, such as randomly cropping and resizing, rotation, Gaussian blurring and color distortion. 
GRACE~\cite{GRACE} extended the SimCLR framework to graph data by generating augmentation views through random link removal and feature masking. Later, GCA~\cite{GCA} prevents the removal of the important links during the stochastic augmentations by designing several link removal mechanisms based on degree centrality, etc. ARIEL~\cite{ARIEL} introduced an adversarial view via a PGD attack for robust training. SPAN~\cite{SPAN} explored spectrum invariance during augmentation views and generated augmentation views by maximizing and minimizing the spectral change. MVGRL~\cite{MVGRL} utilized a local-global contrastive loss to measure the agreement between the neighbor nodes and graph diffusion. DGI~\cite{DGI} generated augmentation views by corrupting the original graph and contrasting the node embeddings between the original graph and the corrupted graphs. BGRL~\cite{BGRL} introduced bootstrapped graph latent training by predicting alternative augmentations of the input without contrasting with negative samples. SPAGCL~\cite{SPAGCL} integrated the node similarity-preserving view and the adversarial view to enhance the adversarial robustness of GCL model. PiGCL~\cite{PiGCL} endeavored to mitigate the implicit conflicts from the negative pairs caused by the InfoNCE loss of GCL to boost its robust performance. GPS~\cite{GPS} deployed a special pooling mechanism to enhance the adversarial robustness of GCL on the graph classification problem.

However, the above-mentioned (robust)-GCL methods contain their unique limitations to some extent. For example, the adversarial view of ARIEL and SPAGCL injects adversarial noises into the feature and topology space of the graph data during training, which will still hinder the robust learning of GCL. The view that minimizes the spectral norm will introduce more inter-class links and produce similar detrimental effects on the graph representation learning. On the other hand, although PiGCL endeavors to improve the robustness of GCL by refining the gradients of the negative pairs, it still cannot essentially erase the adverse impact from the topology space of the graph data. Lastly, the adversarial pooling of GPS is specifically designed to boost the adversarial robustness of GCL on the graph classification task, which falls outside the scope of our topic (robustness of node classification). It is worth noting that unlike other baselines, our work focuses on sanitizing the detrimental effects of the adversarial attacks on the input space, which is supervised by a specially crafted sanitation view of the GCL framework from an information-theoretical perspective.    

%\subsection{Graph Sanitation}
%GASOLINE~\cite{GASOLINE} first proposed the concept of graph sanitation, which addresses the problem of improving the quality of a given graph for a graph mining task. To tackle this problem, the author proposed a bi-level graph structural learning framework that sanitized the noisy graph for better semi-supervised node classification performance via deploying the GNN as its surrogate model. Later, FocusedCleaner~\cite{FocusedCleaner} further extended the graph sanitation problem to the graph security field and proposed a poisoned graph sanitation framework via jointly training a bi-level structural learning module with a victim node detection module to enhance the robustness of GNN. However, these graph sanitation algorithms are specially crafted for the semi-supervised node classification task and are not suitable for the unsupervised GCL. To boost the robustness of the GCL models, we introduce a homophily-preserving learnable sanitation view to scholastically produce more sanitized graph samples during training and thus improve the quality of the node embeddings under various attacking scenarios. 

%% file: sections/methods.tex
\section{Preliminaries}
\label{sec-preliminaries}
In this section, we will provide a brief introduction to the prerequisite knowledge and notations for our work.  

\subsection{Graph Representation Learning}
Given an attribute graph $\mathcal{G}=\{\mathbf{X},\mathbf{A}\}$, where $\mathbf{X}\in\mathbbm{R}^{N\times p}$ is the nodal attribute matrix and $\mathbf{A}\in\mathbbm{R}^{N\times N}$ is the adjacency matrix, graph representation learning aims at training a GNN encoder $f_{\theta}: \mathcal{G}\rightarrow \mathbbm{R}^{N\times d}$ to produce low-dimensional embeddings $\mathbf{H}\in\mathbbm{R}^{N\times d}$ for each node. Then, the pre-trained node embeddings can be fed into a graph-related downstream task such as node classification. $\theta$ summarizes the model parameters to be trained via a pre-defined loss function.     

\subsection{Graph Contrastive Learning Framework}
GCL is a powerful graph representation learning framework that seeks to obtain high-quality node embeddings via maximizing the intractable mutual information between the graph data and its low-dimensional representations, i.e., 
\begin{equation}
    I(G, \mathbf{Z})=\int_{z}\int_{g}p(g,z)\log(\frac{p(g,z)}{p(g)p(z)})dgdz.
    \nonumber
\end{equation}
In order to achieve this goal, a common way is to maximize the similarity between positive pairs (same node with different views) and enlarge the difference between negative pairs (different nodes within the same view and cross different views).  
%More specifically, 
In general, the training of the GCL model has two stages. Firstly, it generates two augmentation views $G_1$ and $G_2$ based on the input graph $G$ via random link removal and feature masking. Then, Two graph samples generated from augmentation views are fed into a shared GNN encoder to produce node embeddings. Finally, a contrastive loss (InfoNCE loss) based on the node embeddings is optimized. 
%{\color{red} this sentence is too long and not clear: More specifically, given a graph $\mathcal{G}$, graph contrastive learning first generates two different graph views $G_{1}=t_{1}(\mathcal{G})$ and $G_{2}=t_{2}(\mathcal{G})$ where $t_1\in\mathcal{T}_1$ and $t_2\in\mathcal{T}_2$ by implementing stochastic graph augmentations $\mathcal{T}_1$ and $\mathcal{T}_2$ on graph data to feed into a shared GNN encoder $f_{\theta}(\cdot)$ and then utilizes a crafted graph contrastive loss, i.e., infoNCE~\cite{infoNCE} to enforce the embeddings of each node for different views to be similar and enlarge the distance between different nodes.} 
The InfoNCE loss is formulated as:
\begin{equation*}
	%\small
    \begin{split}
        &\mathcal{L}_{info}(G_1, G_2)=\frac{1}{2N}\sum_{i=1}^{N}(l(u_i,v_i)+l(v_i,u_i)), \\
        &l(u_i,v_i)= \\
        &-\log\frac{e^{\rho(u_i,v_i)/\tau}}{\begin{matrix}
        \underbrace{e^{\rho(u_i,v_i)/\tau}} \\ \text{positive}
        \end{matrix}+
        \begin{matrix}
        \underbrace{\sum_{k\neq i}e^{\rho(u_i,v_k)/\tau}} \\ \text{inter-view negative}
        \end{matrix}+
        \begin{matrix}
        \underbrace{\sum_{k\neq i}e^{\rho(u_i,u_k)/\tau}} \\ \text{intra-view negative}
        \end{matrix}},
    \end{split}
\end{equation*}
where $u_{i}$ is the $i$-th node for $G_1$, $v_{i}$ is the $i$-th node for $G_2$, and $\tau$ is a temperature hyperparameter. Usually, the similarity function $\rho(\cdot)$ is defined as: 
\begin{equation*}
    \begin{split}
        \rho(u_i,v_i)=cos(g(\mathbf{H}_{1i}), g(\mathbf{H}_{2i}))=\frac{g(\mathbf{H}_{1i})\cdot g(\mathbf{H}_{2i})}{\|g(\mathbf{H}_{1i})\|\cdot\|g(\mathbf{H}_{2i})\|},
    \end{split}
\end{equation*}
which represents the cosine similarity between the projected embeddings of $u_i$ and $v_i$, $g(\cdot)$ is the projection head to enhance the expressive power of the graph representation learning, $\mathbf{H}$ is the node embeddings to be detailed later. We use the two-layered graph convolution~\cite{GNN} as the shared GNN encoder for both two views to get the node embeddings $\mathbf{H}_1=f_{\theta}(G_1)$ and $\mathbf{H}_2=f_{\theta}(G_2)$ respectively. Specifically, $G_1=\{\mathbf{A}_1, \mathbf{X}_1\}$ and $G_2=\{\mathbf{A}_2, \mathbf{X}_2\}$; then we have
\begin{equation}
    \begin{split}
        &\mathbf{H}_l=f_{\theta}(\mathbf{A}_l,\mathbf{X}_l)=\sigma(\hat{\mathbf{A}}_l\sigma(\hat{\mathbf{A}}_l\mathbf{X}_{l}\mathbf{W}^{(1)})\mathbf{W}^{(2)}), \forall  l=1,2, \\
        &\text{where} \ \hat{\mathbf{A}}=\tilde{\mathbf{D}}^{-\frac{1}{2}}\tilde{\mathbf{A}}\tilde{\mathbf{D}}^{-\frac{1}{2}}, \  \tilde{\mathbf{A}}=\mathbf{A+I}, \  \tilde{D}_{ii}=\sum_{j}\tilde{A}_{ij}.
    \end{split}
    \label{eqn-GNN-encoder}
\end{equation}
After that, the graph contrastive learning framework is formulated as follows:
\begin{equation*}
    \begin{split}
        \min_{\theta} \ \mathcal{L}_{info}(t_1(\mathcal{G}), t_2(\mathcal{G}), \theta), \ \text{where} \  t_l\in\mathcal{T}_l, \  l=\{1,2\}.
    \end{split}
\end{equation*}
Then, the optimized node embeddings $\mathbf{H}^{*}=f_{\theta^{*}}(\mathbf{A},\mathbf{X})$ serve as the input to the logistic regression for semi-supervised node classification.

%\subsection{Normalized Cut}
%\label{sec-preliminaries-NormalizedCut}
%Normalized cut loss~\cite{NormCut, maskGVAE, GAP, shi2000normalized, ng2001spectral} is an unsupervised loss function that hcounts  the fraction of the inter-class links for a given graph partition result. Thus, a good graph partition should obtain a lower normalized cut loss, i.e., the majority of the links connect two nodes with the same class. The normalized cut loss is defined as:
%\begin{equation}
%    \begin{split}
%        &\mathcal{L}_{nc}=\frac{1}{K}\Tr((\mathbf{C}^{\top}\mathbf{LC})\oslash(\mathbf{C}^{\top}\mathbf{DC})), \ \mathbf{L}=\mathbf{I}-\hat{\mathbf{A}}.
%    \end{split}
%    \label{eqn-normalized-cut}
%\end{equation}
%Here $\mathbf{C}$ is the cluster assignment matrix, i.e., $C_{ij}$ represents the node $i$ belongs to the cluster $j$, $\mathbf{L}$ is the graph Laplacian matrix, $\oslash$ represents element-wise division. 
%{\color{red}Explain notation in the above equation} 
%In this paper, we use this normalized cut loss as supervision to tune the hyperparameters in the model to eliminate the usage of the validation node labels. Thus, our model is a totally unsupervised learning method within the graph contrastive learning framework.

\section{Vulnerability Analysis of GCL}
\label{sec-vunlerability}
In this section, we present a detailed theoretical and empirical analysis of how existing attacks influence GCL, focusing on key graph properties. These explorations of the vulnerabilities of GCL provide strong insights for our design of the robust GCL framework.  
%We begin with investigating the common principles of graph structural attacks against the GCL framework. 
Without loss of generality, we choose two graph structural attacks, i.e., Mettack~\cite{Mettack} and CLGA~\cite{CLGA} as our graph attackers to simulate real-world attacking scenarios. Specifically, Mettack is the most representative structural attack against GNNs, as other ones~\cite{TopologyAttack} share similar attack loss with Mettack. Meanwhile, CLGA is the up-to-date representative attack method against the GCL framework. Therefore, exploring the underlying principles of Mettack and CLGA can sufficiently reveal the vulnerability of the GCL framework under structural attacks.   

\subsection{Threat Model}
We consider a system consisting of two parties: an attacker and a defender. Specifically, the attacker can manipulate the graph data to degenerate the graph-based learning models deployed by the defender. Conversely, the defender's objective is to recover the quality of the graph representation learning given the poisoned graph to improve node classification performance. To be concrete, we summarize the attacker's goal (Mettack and CLGA as exemplar), knowledge, and capability as follows:

\begin{itemize}
    \item \textbf{Attacker's goal}: The attacker’s goal is assumed to be decreasing the classification accuracy of a node classification problem achieved after training on poisoned graph data manipulated by the attacker. 
    \item \textbf{Attacker's knowledge}: The attacker can have different levels of knowledge of the data, model, and model parameters. For Mettack and CLGA, the attacker has no knowledge about the model with its parameters. In the meanwhile, the attacker can observe all nodes’ attributes, the graph structure. The main difference between Mettack and CLGA is that Mettack can observe the training labels. However, CLGA cannot query the label information of the graph data since it is an unsupervised attack method.
    \item \textbf{Attacker's capability}: In order to achieve unnoticeable attacks, Mettack and CLGA both impose a budget constraint $\Delta$ to restrict the number of perturbations on the topology space of the graph data, i.e., $\frac{1}{2}\|\mathbf{A}-\mathbf{A}^{p}\|\leq\Delta$, where $\mathbf{A}^{p}$ is the poisoned adjacency matrix.
    \item \textbf{Defender's knowledge}: In reality, the defender cannot acquire the attacking scenarios (such as the attack type, attack degree, and victim nodes or edges, etc.), and the label information of the graph data. However, the defender can have full knowledge of the given graph data's information (the node attribute matrix and adjacency matrix) as well as the model's information (model structure and model parameters).
\end{itemize}

\subsection{Disturbance on Mutual Information}
\label{sec-analyze-infoNCE}
\begin{figure}[h]
	\centering
	\includegraphics[width=0.48\textwidth,height=2.5cm]{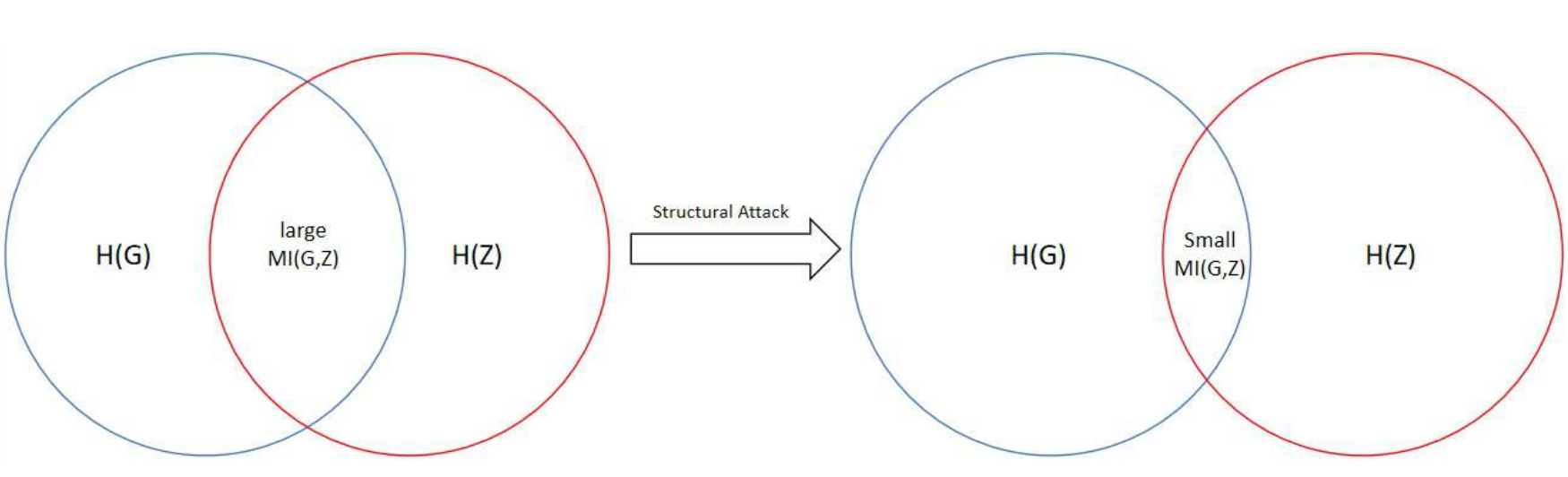}
	\caption{The blue and green circles represent the information entropy of the graph $G$ and embedding $\mathbf{Z}$, with the shaded area indicating the mutual information.} %The structural attacks aim at minimizing the mutual information between them.}
\end{figure}

The powerful graph representation learning of the GCL framework highly relies on capturing the mutual information between the graph and its representations in low-dimensional space. To explore the failure of the GCL framework under the adversarial attack scenario, we start by analyzing the vulnerability of the GCL framework from an information theoretical perspective. For graph contrastive learning, a common way to curve the intractable mutual information is to utilize the InfoNCE object:
\begin{lemma}[InfoNCE~\cite{infoNCE}]
\label{lemma-infoNCE}
The InfoNCE object~\cite{GCA} $I_{info}(\cdot)=-\mathcal{L}_{info}(\cdot)$ is the lower bound approximation of the intractable mutual information between the graph and its representations $I(G;\mathbf{Z}=f_{\theta}(G))$.
%{\color{red} Is this function I (X,G) correct?}
\end{lemma}

Then, we provide the theorem to curve the relationship between the graph attacker and the mutual information $I(G;\mathbf{Z})$.
\begin{theorem}
\label{theorem-attack-info}
    The %graph attacker (Mettack or CLGA) 
    two attacks, Mettack and CLGA, maliciously degenerate the GCL's performance by diminishing the mutual information between the graph and its representations. 
\end{theorem}
\begin{proof}
    We denote the poisoned graph generated by the graph attacker as $G^{p}=\{\mathbf{X},\mathbf{A}^{p}\}$, the corresponding poisoned node embeddings as $\mathbf{Z}^{p}=f_{\theta}(G^{p})$, where $f_{\theta}(\cdot)$ is the GNN encoder. Given the InfoNCE object $I_{info}(\cdot)=-\mathcal{L}_{info}(\cdot)$, if $G^{p}$ is generated by CLGA, we have
    \begin{equation}
        G^{p}=\arg\min_{G} \ I_{info}(t_{1}(G); t_{2}(G)),
    \end{equation}
    since the InfoNCE loss is the objective function of CLGA. Then, we have 
    \begin{equation}
    \label{eqn-attack-Info}
    \begin{split}
        I_{info}(t_{1}(G); t_{2}(G))>I_{info}(t_{1}(G^{p}); t_{2}(G^{p})), 
    \end{split}
    \end{equation}
    based on Lemma~\ref{lemma-infoNCE}, we have:
    \begin{equation}
    \label{eqn-attack-MI}
    \begin{split}
        I(G; \mathbf{Z}=f_{\theta}(G))>I(G^p; \mathbf{Z}^{p}=f_{\theta}(G^p)).
    \end{split}
    \end{equation}
    If $G^{p}$ is generated by Mettack, we have
    \begin{equation}
        \begin{split}
            G^p=\arg\max_{G} \ \mathcal{L}_{CE}(\mathbf{Z}^{p}, \mathbf{Y}), 
        \end{split}
    \end{equation}
    where $\mathcal{L}_{CE}$ is the cross-entropy loss. We then partition the $\mathcal{L}_{CE}$ as: 
    \begin{equation}
    \label{eqn-proof-CE}
        \begin{split}
            \mathcal{L}_{CE}&=H(\mathbf{Y})+D_{KL}(p_Y\|p_Z) \\
            &=H(\mathbf{Y}|\mathbf{Z})+I(\mathbf{Y};\mathbf{Z})+D_{KL}(p_Y\|p_Z) \\
            &=H(\mathbf{Y}|\mathbf{Z})+D_{KL}(p_{YZ}\| p_Y p_Z)+D_{KL}(p_Y\|p_Z),
        \end{split}
    \end{equation}
    Next, we build the relationship between mutual information $I(G;\mathbf{Z})$ with the conditional entropy $H(\mathbf{Y}|\mathbf{Z})$ as 
    \begin{equation}
    \label{eqn-proof-MI}
        \begin{split}
            I(G;\mathbf{Z})&=H(G)-H(G|\mathbf{Z})\\
            &=H(G)-H(\mathbf{Y}|\mathbf{Z})+H(\mathbf{Y}|G,\mathbf{Z})-H(G|\mathbf{Y},\mathbf{Z}),
        \end{split}
    \end{equation}
    Based on Eqn.~\ref{eqn-proof-CE} and \ref{eqn-proof-MI}, we have
    \begin{equation}
    \label{eqn-proof-MI-CE}
        \begin{split}
            I(G; \mathbf{Z})&=H(G)-\mathcal{L}_{CE}+D_{KL}(p_{YZ}\|p_{Y}p_{Z})\\
            &+D_{KL}(p_{Y}\|p_{Z})+H(\mathbf{Y}|G,\mathbf{Z})-H(G|\mathbf{Y},\mathbf{Z}).
        \end{split}
    \end{equation}
    Based on Eqn.~\ref{eqn-proof-MI-CE}, we observe that the mutual information $I(G;\mathbf{Z})$ is negatively correlated with $\mathcal{L}_{CE}$. Since Mettack can increase the $\mathcal{L}_{CE}$:
    \begin{equation}
        \begin{split}
            \mathcal{L}_{CE}(\mathbf{Z},\mathbf{Y})<\mathcal{L}_{CE}(\mathbf{Z}^{p},\mathbf{Y}),
        \end{split}
    \end{equation}
    then, Eqn.~\ref{eqn-attack-MI} also holds for the poisoned graph generated by Mettack. 
\end{proof}

\begin{figure}[h]
	\begin{center}
		\subfloat[Mettack]{\includegraphics[width=0.24\textwidth,height=3.2cm]{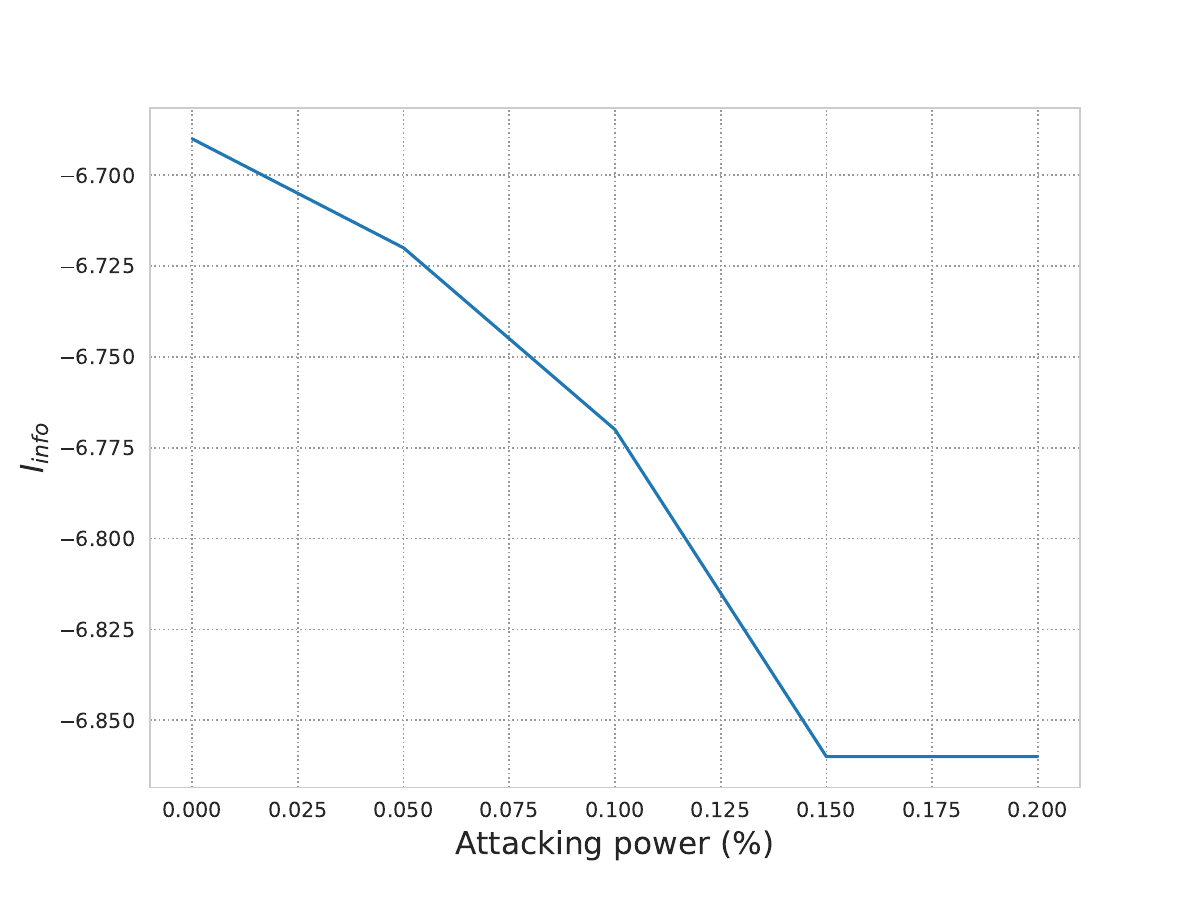}}
        \hfill		
        \subfloat[CLGA]{\includegraphics[width=0.24\textwidth,height=3.2cm]{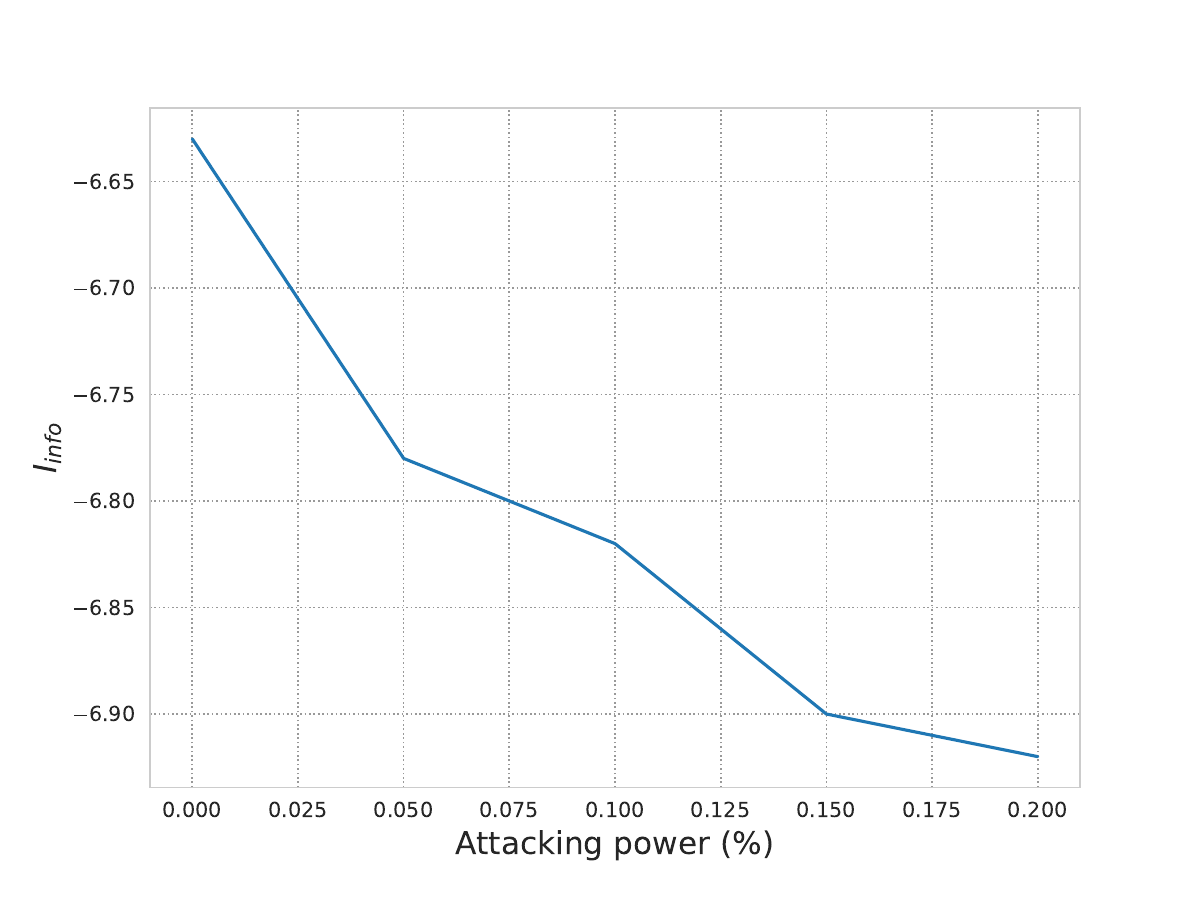}}
	\end{center}
	\caption{$I_{InfoNCE}$ for Cora dataset attacked by Mettack and CLGA with varying attacking powers.}
	\label{fig-attack-vs-info}
\end{figure}

Theorem \ref{theorem-attack-info} demonstrates that by maliciously altering the graph topology, attacks can degrade the mutual information estimation between the graph data and its representations. On the other hand, GCL's high-quality node representations are derived from maximizing this mutual information. Therefore, the GCL model may only be able to optimize the mutual information to a sub-optimal level, resulting in the production of low-quality embeddings. We further conduct empirical studies to verify whether the structural attacks may maliciously influence mutual information. We first implement Mettack and CLGA individually to poison the clean graph (Cora~\cite{CitationNetwork}) under different attacking scenarios. Next, we train a simple GCL model, i.e., GRACE~\cite{GRACE} on the poisoned graphs and report the converged InfoNCE object in Fig.~\ref{fig-attack-vs-info}. The results demonstrate that the structural attacks on GRACE indeed degrade the mutual information estimation between the graph and its representations.

\subsection{Disturbance on Graph Homophily}
\label{sec-analyze-homo}
It has been widely explored that structural attacks against GNNs (Mettack) will diminish the homophily level of the clean graphs~\cite{GCNJaccard,ProGNN,GNNGUARD,GNNhomophily,SimPGCN}. In this paper, we also explore whether the above-mentioned phenomenon is a common principle for graph structural attacks, including Mettack and CLGA. 
In fact, the graph homophily is defined as:
\begin{definition}[Graph homophily on label space~\cite{homophily}]
    The graph homophily based on node labels is given by:
    \begin{equation}
        \begin{split}
        h_{y}=\frac{1}{\mathcal{E}}\sum_{(v_{i},v_{j})\in\mathcal{E}}\mathbbm{1}_{(y_{i}=y_{j})}, \label{eqn-homo-label}
        \end{split}
    \end{equation}
    where $y_{{i}}$ is the label of node $v_{i}$, $\mathcal{E}$ is the link set.
\end{definition}

Besides using node labels, the graph homophily can also be measured by the Euclidean distance between the attribute vectors of the connected nodes~\cite{XLX,ProGNN}. It is common to use a trace formula to reformulate such distances in matrix format. We then present another metric on graph homophily as follows:
\begin{definition}[Graph homophily on feature space~\cite{ProGNN}]
    The graph homophily based on node attributes is given by:
    \begin{equation}
        \begin{split}
            \delta_{x}=\frac{1}{2}\sum_{i,j=1}^{N} A_{ij}(\frac{\mathbf{x}_{i}}{\sqrt{d_{i}}}-\frac{\mathbf{x}_{j}}{\sqrt{d_{j}}})^2=tr(\mathbf{X}^{T}\mathbf{L}\mathbf{X}).
        \end{split}
    \end{equation}
    where $\mathbf{L}=\mathbf{I}-\mathbf{\tilde{D}}^{-\frac{1}{2}}\mathbf{\tilde{A}}\mathbf{\tilde{D}}^{-\frac{1}{2}}$ is the normalized graph Laplacian, $\mathbf{\tilde{A}}$ is the adjacency matrix with self-loop. 
\end{definition}
A higher $h_y$ (or a lower $\delta_x$) indicates that the graph tends to be more homophilous. We conduct empirical studies to verify whether \textit{structural attacks will diminish the graph homophily} is a common principle for Mettack and CLGA. Fig.~\ref{fig-attack-vs-homo} presents the relationship between the node classification accuracy of GRACE~\cite{GRACE} and the graph homophily metrics ($h_{y}$ and $\delta_{x}$) on poisoned graphs with varying attacking powers. It is clearly observed that the accuracy is positively correlated with $h_y$ and is negatively correlated with $\delta_x$. That is, both attacks will significantly diminish the graph homophily.

\begin{figure}[h]
	\begin{center}
		\subfloat[Mettack]{\includegraphics[width=0.24\textwidth,height=3.2cm]{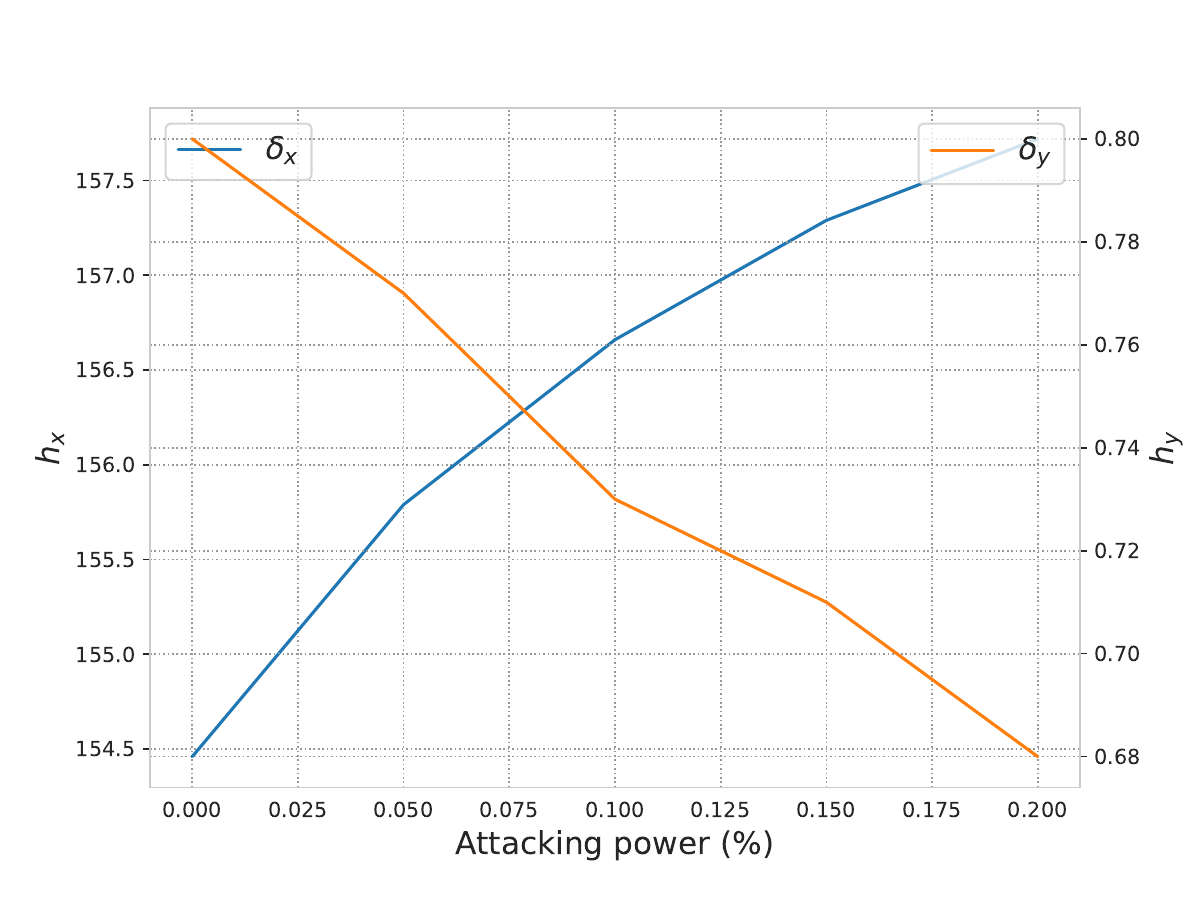}}
        \hfill		
        \subfloat[CLGA]{\includegraphics[width=0.24\textwidth,height=3.2cm]{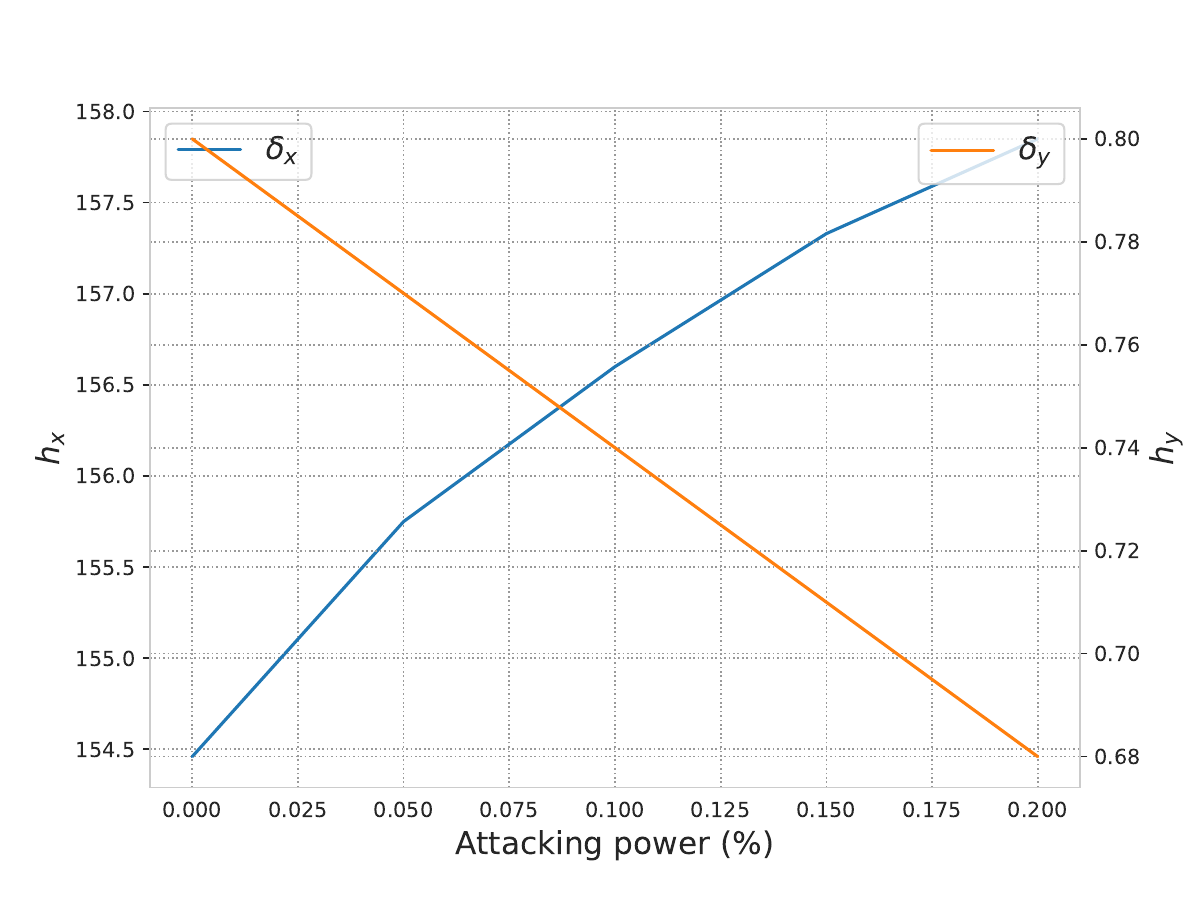}}
	\end{center}
	\caption{Graph homophily $\delta_x$ and $h_y$ for Cora dataset attacked by Mettack and CLGA with varying attacking powers.}
	\label{fig-attack-vs-homo}
\end{figure}

Overall, the findings in Sec.~\ref{sec-analyze-infoNCE} and \ref{sec-analyze-homo} that adversarial attacks will provide malicious effects to the mutual information and the graph homophily can serve as strong insights for us to design the defense strategy for the GCL framework. That is, restoring the diminished mutual information and graph homophily is promising to sanitize the malicious effects and thus achieve adversarial robustness.   

\section{Defense Methodology}
\label{Sec-Methodology}
\begin{figure*}
    \centering
    \includegraphics[width=0.85\textwidth,height=6.cm]{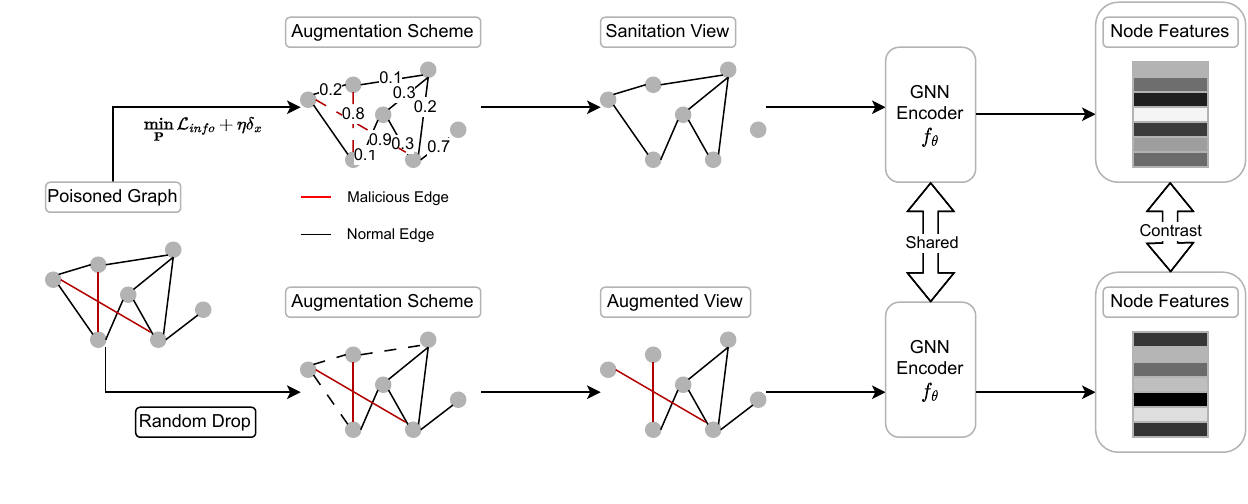}
    \caption{The overall architecture of \textbf{GCIR}.}
    \label{fig-overview}
\end{figure*}

\subsection{Overview} 
The proposed model's overall framework is illustrated in Fig.~\ref{fig-overview}. The model comprises a data augmentation phase and a contrastive learning phase. In the first phase, we introduce a learnable sanitation view to restore the diminished mutual information and graph homophily. In the second phase, the sanitation view is contrasted with the second augmented view, which is produced by random link removal and feature masking. Finally, we optimize a cross-view contrastive learning objective and the learnable sanitation view in an end-to-end manner. In this way, the learned node embeddings will restore the diminished mutual information and achieve adversarial robustness. 

\subsection{Sanitation View Generation}
As previously mentioned in Sec.~\ref{Sec-intro}, we endeavor to achieve the robustness of GCL via restoring the contaminated mutual information after attacks. Unlike the traditional adversarial training method that achieves robustness by further ``polluting" the poisoned graph data, we instead choose to ``sanitize" the poisoned graph data to rectify the mutual information distortion caused by the attacks during the data augmentation phase. To achieve this goal, we introduce the sanitation view with a stochastic edge-dropping scheme to sanitize malicious effects provided by attacks. To be concrete, the stochastic edge-dropping scheme of the sanitation view can adaptively restore the mutual information during graph data augmentation and thus achieve robustness. Specifically, we use an edge vector $\mathbf{E}$ to represent all the existing edges in a graph. Then, we define a manipulation vector $\mathbf{M} = \{0,1\}^{|\mathbf{E}|}$ that follows the Bernoulli distribution parameterized by $\mathbf{P}$; that is $\mathbf{M}\sim Ber(\mathbf{P})$. For a particular entry $\mathbf{M}_e$, we will remove the corresponding edge $e$ if $\mathbf{M}_e = 1$, which is controlled by the probability $\mathbf{P}_e$. Now, the sanitation view can be regarded as a mapping $\mathcal{T}_{\mathbf{P}} (\mathbf{E}) \rightarrow \mathbf{E}_{\mathcal{S}}$ parameterized by $\mathbf{P}$, which takes a graph with edge vector $\mathbf{E}$ as input and produces a sanitized view $\mathbf{E}_{\mathcal{S}}$. Formally, we have:
%In our framework, the sanitation view focuses on implementing stochastic topology augmentation to automatically sanitize the candidate inter-class links injected by the graph attacker in the poisoned graph. To this end, we focus on the edge-dropping scheme for the sanitation view and parametrize it with the edge-dropping probability matrix which follows the Bernoulli distribution. More specifically, we define the sanitation view as a crafted mapping to the edge vector $\mathbf{E}$ of the input graph and a manipulation vector $\mathbf{Z}$ with each entry following the Bernoulli distribution:
\begin{equation}
    \begin{split}
        \mathbf{E}_{\mathcal{S}}=\mathcal{T}_{\mathbf{P}}(\mathbf{E}) \triangleq \mathbf{E}\odot(\mathbf{1}_{|\mathbf{E}|}-\mathbf{M}), \ \ \mathbf{M}\sim Ber(\mathbf{P}),
    \end{split}
    \label{eqn-san-sampling}
\end{equation} 
where $\odot(\cdot)$ is the element-wise product operation, $\mathbf{1}_{|\mathbf{E}|}$ is an all-one vector with length equal to the edge number ${|\mathbf{E}|}$, $\mathbf{P}=\{p_{e}\}_{e=1}^{|\mathbf{E}|}$ is the parameters of the sanitizer. 
%Here $\mathbf{Z}_{e}=1$ means the sanitation view removes the corresponding link $e$ and vice versa. 

Thus, the remaining task amounts to choosing the proper parameters $\mathbf{P}$ for the sanitizer $\mathcal{T}_{\mathbf{P}}$ so that the generated sanitation views can facilitate contrastive learning to achieve adversarial robustness. Next, we show how to learn the parameters $\mathbf{P}$. 

\subsubsection{Learning of Robust GCL}
The key challenges to learning $\mathcal{T}_{\mathbf{P}}$ are two-fold: choosing the proper objective to guide learning, and integrating learning into the framework of GCL. 
We design an objective function that contains two components, i.e., InfoNCE loss and feature smoothness. Intuitively, forcing the learning of the sanitation view to decrease the InfoNCE loss and feature smoothness can restore the diminished mutual information and graph homophily during training. Specifically, the objective function is :
\begin{equation}
    \begin{split}
        \min_{\mathbf{P}\in\mathcal{S}_1, \theta} \ &\mathcal{L}=\mathcal{L}_{info}(f_{\theta}(\mathbf{E}_{\mathcal{S}}), f_{\theta}(t_\mathcal{R}(\mathbf{E})))+\eta\Tr(\mathbf{X}^{\top}\mathbf{L}_{\mathcal{S}}\mathbf{X}), \\ 
        &\text{s.t.} \ \mathbf{E}_{\mathcal{S}}=\mathbf{E}\odot(\mathbf{1}_{|\mathbf{E}|}-\mathbf{M}), \ \mathbf{M}\sim Ber(\mathbf{P}),\\
        & \quad \  t_{\mathcal{R}}(\mathbf{E})\in\mathcal{T}_{\mathcal{R}}, \ \mathbf{L}_{\mathcal{S}}=\mathbf{I}-\tilde{\mathbf{D}}_{\mathcal{S}}^{-\frac{1}{2}}(\mathbf{A}_{\mathcal{S}}+\mathbf{I})\tilde{\mathbf{D}}_{\mathcal{S}}^{-\frac{1}{2}},
    \end{split}
    \label{eqn-san-obj-2}
\end{equation} 
where the first view $\mathbf{E}_{\mathcal{S}}$ is the learnable sanitation view, the second augmented view $t_\mathcal{R}(\mathbf{E})$ is generated from typical random augmentation scheme $\mathcal{T}_R$ that will randomly drop edges and feature masking with a fixed probability, $\Tr(\cdot)$ represents the trace of a matrix  and $\mathbf{S}_1$ is a constrained space of the sanitation operation with $\epsilon$ controlling the sanitation degree: 
\begin{equation}
    \begin{split}
        \mathbf{S}_1=\{s|s\in[0,1]^{E}, \|s\|_1\leq \epsilon\}.
    \end{split}
\end{equation}
$\mathbf{L}_{\mathcal{S}}$ is the graph Laplacian based on the sanitation view $\mathbf{E}_{\mathcal{S}}$.

%In this way, the augmented graph derived from the sanitation view can degrade the mutual information estimation and the graph homophily guided by the crafted objection. 

In our design, the training objective of the sanitation view is embedded into that of the GCL. That is, we train the learnable sanitation view together with the contrastive learning, using a hyperparameter $\eta$ to control the relative importance of optimizing InfoNCE loss and the graph homophily $\delta_{x}$. %
Then, the learnable sanitation view samples $\mathbf{E}_{\mathcal{S}}$ based on Eqn.~\ref{eqn-san-sampling} for each iteration. After that, the shared GNN encoder $f_{\theta}(\cdot)$ converts the graph data from the two views to low-dimensional node embedding matrices and obtains the training loss $\mathcal{L}$. During the backward pass, we compute the gradient of $\mathcal{L}$ with respect to the parameters $\mathbf{P}$ and $\theta$ and update them via gradient descent. Consequently, the sanitation view can restore the diminished mutual information guided by the sanitation view's objection. 

\subsubsection{Learning of Stochastic Edge-dropping Scheme}
However, the above-mentioned learning scheme still encounters a unique challenge: It is difficult to directly optimize the stochastic edge-dropping scheme of the sanitation view due to the undifferentiable sampling procedure. To tackle this issue, we refer to the Gumbel-Softmax re-parametrization technique~\cite{GumbelSoftmax} to relax the discrete Bernoulli sampling procedure to a differentiable Bernoulli approximation. Using the relaxed sample, we implement the straight-through estimator~\cite{STE} to discretize the relaxed samples in the forward pass and set the gradient of discretization operation to $1$, i.e., the gradients are directly passed to the relaxed samples rather than the discrete values. Hence, we reconstruct the mapping from $\mathbf{P}$ to $\mathbf{M}$ as:
\begin{equation}
    \begin{split}
        \mathbf{M}_{e}=\lfloor\frac{1}{1+e^{-(\log\mathbf{P}_e+g)/\tau}}+\frac{1}{2}\rceil,
    \end{split}
\end{equation}
where $\lfloor\cdot\rceil$ is the rounding function, $g\sim Gumbel(0,1)$ is a standard Gumbel random variable with zero mean and the scale equal to $1$. Then, we can reformulate Eqn.~\ref{eqn-san-sampling} to:
\begin{equation}
	%\small
    \begin{split}
        &\mathbf{E}_{\mathcal{S}}=\mathbf{E}\odot(\mathbf{1}_{|E|}-\lfloor\frac{1}{1+e^{-(\log\mathbf{P}+g)/\tau}}+\frac{1}{2}\rceil), \\ 
        &\text{where} \ \mathbf{P}\in\mathcal{S}_1=\{\mathbf{P}_{e}|\sum_{e}\mathbf{P}_{e}\leq \epsilon, \mathbf{P}_{e}\in[0,1] \ \forall e\in\{1,..,E\}\}. 
    \end{split}
    \label{eqn-san-Gumbel-sampling}
\end{equation}  
Formally, to optimize the parameters $\mathbf{P}$ under the constraint $\mathcal{S}_{1}$ defined in Eqn.~\ref{eqn-san-Gumbel-sampling}, we deploy the projection gradient descent to optimize the relaxed constrained optimization problem introduced in Eqn.~\ref{eqn-san-Gumbel-sampling}:
\begin{equation}
    \begin{split}
        \mathbf{P}^{(t)}=\textstyle\prod_{\mathcal{S}_1}(\mathbf{P}^{(t-1)}-\alpha\nabla_{\mathbf{P}^{(t-1)}}\mathcal{L}(t_1(\mathcal{G}), t_2(\mathcal{G}), \theta))
    \end{split}
    \label{eqn-PGD}
\end{equation}
at the $t$-th iteration, where $\alpha$ is the learning rate of the projection gradient descent, $\nabla_{\mathbf{P}}\mathcal{L}(t_1(\mathcal{G}),t_2(\mathcal{G}),\theta)$ denotes the gradient of the loss defined in Eqn.~\ref{eqn-san-obj-2}, $\prod_{\mathcal{S}_1}(\cdot)$ is the projection operator to project the updated parameters to satisfy the constraint. Referring to \cite{TopologyAttack}, the projection operator $\prod_{\mathcal{S}_1}(\cdot)$ has the closed-form solutions:
\begin{equation*}
    \begin{split}
        \textstyle\prod_{\mathcal{S}_1}(\mathbf{P})=\begin{cases}
                                            \prod_{[0,1]}[\mathbf{P}-\mu\mathbf{1}_{E}],&\mbox{if } \mu>0\mbox{ and }
                                            \\ &\sum_{e}\prod_{[0,1]}[\mathbf{P}-\mu\mathbf{1}_{E}]=\epsilon, \\ \\ 
                                            \prod_{[0,1]}[\mathbf{P}],&\mbox{if } \sum_{e}\prod_{[0,1]}[\mathbf{P}]\leq\epsilon, 
                                        \end{cases}
    \end{split}
\end{equation*}
where $\prod_{[0,1]}[\mathbf{P}]$ clips the parameters vector $\mathbf{P}$ into the range $[0,1]$. 
The bisection method~\cite{bisection} is used over $\mu\in\{\min(\mathbf{P}-\mathbf{1}_{E}), \max(\mathbf{P})\}$ to find the solution to $\sum_{e}\prod_{[0,1]}[\mathbf{P}-\mu\mathbf{1}_{E}]=\epsilon$ with the convergence rate $\log_{2}[(\max(\mathbf{P})-\min(\mathbf{P}-\mathbf{1}_{E}))/\xi]$ with $\xi$-error tolerance~\cite{liu2015sparsity}. We present the details of algorithm in Alg.~\ref{alg-RGCL}.
\begin{algorithm}[htp]
    \caption{\textbf{GCIR}}
    \flushleft
    \label{alg-RGCL}
    \textbf{Input}: Poisoned graph $\mathcal{G}=\{\mathbf{E},\mathbf{X}\}$, $\eta$.\\
        \textbf{Parameters}: Sanitation probability $\mathbf{P}$, GNN encoder $f_{\theta}(\cdot)$ with weight matrices $\theta=\{\mathbf{W}^{(1)}, \mathbf{W}^{(2)}\}$.\\
    \textbf{Output}: Sanitized node embeddings matrix $\mathbf{H}_1$.\\
    \begin{algorithmic}[1] %[1] enables line numbers
        \STATE Let $t=0$, initialize parameters $\mathbf{P}=\mathbf{P}^{0}$, $\theta=\theta^{0}$.
		\WHILE{$t\leq T$}
        \STATE Sample $g\sim Gumbel(0,1)$.
        \STATE Sample sanitation view $\mathcal{G}_{1}^{t}$ from $\mathbf{E}_{\mathcal{S}}^{t}=\mathbf{E}\odot(\mathbf{1}_{E}-\lfloor\frac{1}{1+e^{-(\log\mathbf{P}^{t}+g)/\tau}}+\frac{1}{2}\rceil)$.
        \STATE Sample another view $\mathcal{G}_{2}^{t}$ by randomly dropping links. 
        \STATE Compute graph Laplacian $\mathbf{L}^{t}=\mathbf{I}-(\tilde{\mathbf{D}}^{t})^{-\frac{1}{2}}\tilde{\mathbf{A}}^{t}(\tilde{\mathbf{D}}^{t})^{-\frac{1}{2}}$.
        \STATE Compute loss $\mathcal{L}=\mathcal{L}_{info}(\mathcal{G}_{1}^{t}, \mathcal{G}_{2}^{t})+\eta\Tr(\mathbf{X}^{\top}\mathbf{L}^{t}\mathbf{X})$.
        \STATE Projection gradient descent and update $\mathbf{P}^{t}$ from Eqn.~\ref{eqn-PGD}.
        \STATE Gradient descent and update $\theta^{t}$ based on $\mathcal{L}$.
        \STATE Compute pseudo normalized cut loss $\mathcal{L}_{pnc}^{t}$ from Eqn.~\ref{eqn-pseudo-normalized-cut}. 
        \ENDWHILE
        \STATE Get the best iteration $t^{*}$ with the minimum value $\mathcal{L}_{pnc}^{t^{*}}$.
        \RETURN Sanitized node embeddings $\mathbf{H}_1=f_{\theta^{t^{*}}}(\mathbf{A}^{t^{*}},\mathbf{X})$.
	\end{algorithmic}
\end{algorithm} 

\subsection{Unsupervised Tuning Strategy}
\label{sec-unsupervised-tuning}
Strictly speaking, the whole training procedure of the GCL models is label-free, that is, the data analyst cannot query the label's information from the downstream tasks to supervise the training of GCL, including the parameter tuning phase. Hence, it is vital to craft a totally unsupervised tuning strategy to select the best hyperparameters. We herein concentrate on designing a new metric in an unsupervised manner without additional training or fine-tuning. 

In particular, normalized cut loss~\cite{NormCut} is an unsupervised loss function that counts the fraction of the inter-class links for a given graph partition result. Thus, a good graph partition should obtain a lower normalized cut loss, i.e., the majority of the links connect two nodes with the same class. The normalized cut loss is defined as:
\begin{equation}
    \begin{split}
        &\mathcal{L}_{nc}=\frac{1}{K}\Tr((\mathbf{C}^{\top}\mathbf{LC})\oslash(\mathbf{C}^{\top}\mathbf{DC})), \ \mathbf{L}=\mathbf{I}-\hat{\mathbf{A}}.
    \end{split}
    \label{eqn-normalized-cut}
\end{equation}
Here $\mathbf{C}$ is the cluster assignment matrix, i.e., $C_{ij}$ represents the node $i$ belongs to the cluster $j$, $\mathbf{L}$ is the graph Laplacian matrix, $\oslash$ represents element-wise division. We propose to utilize this normalized cut loss as the supervision for searching the vital parameter $\eta$ in Eqn.~\ref{eqn-san-obj-2}. The intuition is that if the generated node embeddings have high quality, the clustering result should be accurate, leading to a low loss value. In fact, several previous works~\cite{maskGVAE} have utilized this loss as supervision in different ways. 
Specifically, during the $t$-th iteration, we obtain the sanitized adjacency matrix $\mathbf{A}_{\mathcal{S}}^{(t)}$ (corresponding with $\mathbf{E}_{\mathcal{S}}^{(t)}$ in Eqn.~\eqref{eqn-san-Gumbel-sampling}). 
Then we feed it into a two-layered GNN encoder to the sanitized node embeddings $\mathbf{H}_{\mathcal{S}}^{(t)}$
%=f_{\theta}(\mathbf{A}_{\mathcal{S}}^{(t)},\mathbf{X})$ from Eqn.~\ref{eqn-GNN-encoder} 
and get the pseudo normalized cut:
\begin{equation*}
	\small
    \begin{split}
        \mathcal{L}_{pnc}^{(t)}=\Tr((\sigma(\mathbf{H}_{1}^{(t)})^{\top}\mathbf{L}^{(t)}\sigma(\mathbf{H}_{1}^{(t)}))\oslash(\sigma(\mathbf{H}_{1}^{(t)})^{\top}\mathbf{D}^{(t)}\sigma(\mathbf{H}_{1}^{(t)})),
    \end{split}
    \label{eqn-pseudo-normalized-cut}
\end{equation*}
where $\sigma(\cdot)$ is the sigmoid function. We record the value of $\mathcal{L}_{pnc}^{(t)}$ for each iteration during training and select the model with the minimum $\mathcal{L}_{pnc}$ as our final result. We denote the best metric as $\mathcal{L}_{pnc}^{t^{*}}$. It is worth noting that the computing of the pseudo normalized cut $\mathcal{L}_{pnc}$ does not rely on the labels, which is suitable for the unsupervised training of GCL. 

\begin{wrapfigure}{r}{0.25\textwidth}
    \centering
    \includegraphics[width=0.25\textwidth,height=2.5cm]{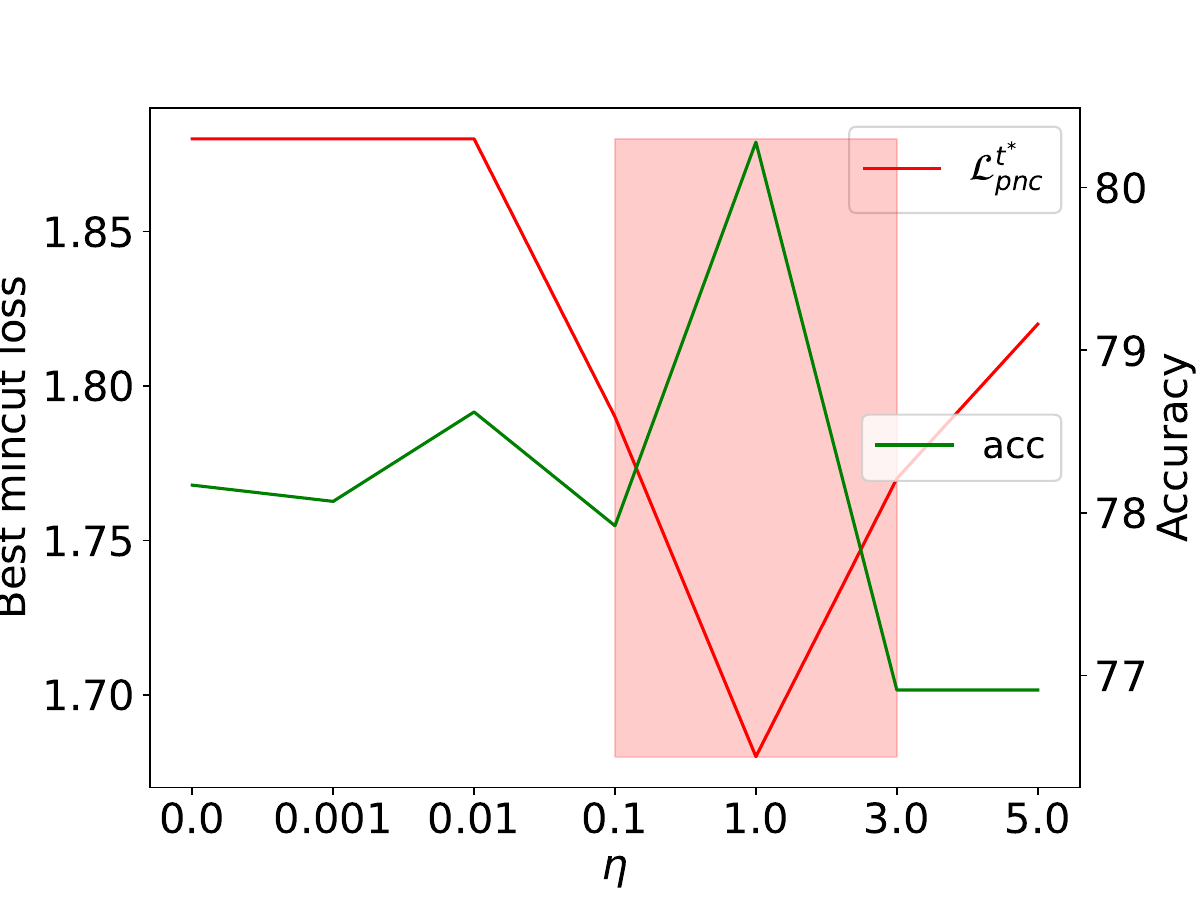}
    \caption{$\mathcal{L}_{pnc}^{t^{*}}$ v.s. Acc.}
    \label{fig-pnc-acc}
\end{wrapfigure}

Besides, we do not introduce additional parameters to ensure the time efficiency of computing $\mathcal{L}_{pnc}$.  

%\begin{figure}
%    \centering
%    \includegraphics[width=0.35\textwidth,height=4.cm]{figures/CoraML-CLGA-0.1-mincut-vs-acc.pdf}
%    \caption{$\mathcal{L}_{pnc}^{t^{*}}$ v.s. Acc.}
%    \label{fig-pnc-acc}
%\end{figure}

In order to verify the effectiveness of the crafted unsupervised tuning strategy, we present the values of $\mathcal{L}_{pnc}^{t^{*}}$ and node classification accuracy with varying $\eta$ in Fig.~\ref{fig-pnc-acc}. The results indicate that in most cases $\mathcal{L}_{pnc}^{t^{*}}$ is negatively related to the node classification accuracy, especially when the value of $\mathcal{L}_{pnc}^{t^{*}}$ falls into the neighborhood of the minimum value, i.e., $\eta\in[0.1, 3]$ (the Pearson correlation coefficients $r(\mathcal{L}_{pnc}^{t^{*}}, accuracy)=-0.89$). This phenomenon demonstrates that it is reasonable to tune $\eta$ via searching for the best $\mathcal{L}_{pnc}^{t^{*}}$, since the GCL model that achieves the minimum $\mathcal{L}_{pnc}^{t^{*}}$ is likely to get the best accuracy.

%% file: sections/experiments.tex
\section{Experiments}
\label{Sec-exp}
In this section, we introduce the experimental settings and conduct comprehensive experiments to validate the effectiveness and efficiency of our robust model. Specifically, we outline the questions that we seek to answer below: 
\begin{itemize}
    \item \textbf{RQ1:} Comparing with baseline models, does our model achieve better robust performances?
    \item \textbf{RQ2:} How do different components of the model contribute to its performance?
    \item \textbf{RQ3:} How do the hyperparameters in the proposed model impact the quality of the node embeddings?
\end{itemize}

\subsection{Dataset \& Settings}
\begin{table}[h]
	\centering
	\caption{Dataset statistics.}
	\label{tab-dataset}
	\resizebox{0.8\columnwidth}{!}{%
		\begin{tabular}{c|cccc}
			\toprule[1.pt]
			Datasets & \#Nodes & \#Edges & \#Classes & \#Features \\
			\hline
			Cora     & $2708$  & $5278$  & $7$   & $1433$\\ 
			Citeseer & $3327$  & $4552$  & $6$   & $3703$\\
			Cora-ML & $2995$  & $8416$ & $7$ & $2879$\\
                Photo & $7650$ & $119081$ & $8$ & $745$\\
                Computers & $13752$ & $245861$ & $10$ & $767$\\
                WikiCS & $11701$ & $215863$ & $10$ & $300$\\
			\bottomrule[1.pt]
		\end{tabular}
	}
\end{table}

We use six popular benchmark graph datasets: Cora, CiteSeer, Cora-ML, Photo, Computers and WikiCS~\cite{CitationNetwork, WikiCS, Amazon} for evaluation. To validate the quality of the node embeddings generated by GCL effectively, we train a multi-class logistic regression model with Adam optimizer~\cite{Adam} and randomly split the node labels into three groups: training, validation and testing set with the ratio as $10\%$, $10\%$ and $80\%$. 
%The logistic regression model has a learning rate of $0.01$ and $5000$ training iterations. 
%For graph clustering, the K-means clustering~\cite{KMeans} algorithm is run based on the pre-trained node embeddings derived from GCL models and we report the normalized mutual information (NMI) to quantify the robustness of the graph clustering performances.
%A two-layered GNN encoder is deployed for all the GCL models. The embedding dimension of the first layer is $128$ for all the datasets and the embedding dimension of the second layer is $32$ for Cora, CiteSeer and Cora-ML and $128$ for Photo, Computers and WikiCS. We train the GNN encoder parameters $\theta$ via Adam optimizer with the learning rate equal to $0.0005$ and the training epochs equal to $1000$. 
We run all the models $10$ times with different seeds and report the mean value for a fair comparison. The vital hyperparameter $\eta$ is tuned based on the unsupervised tuning strategy without querying the node labels. 
%All the models are deployed on an Intel 10C20T Core i9-10850K CPU with GIGABYTE RTX3090 24GB GPU.
\begin{table*}[h]
\centering
    \caption{Performances on node classification against Mettack.}
	\label{tab-defend-Mettack}
	\resizebox{0.9\textwidth}{!}{%
\begin{tabular}{c|ccccccccccc}
\toprule[1.pt]
 Dataset & $\frac{B}{|E|}$ & PiGCL & BGRL       & DGI        & MVGRL               & GRACE       & ARIEL       & GCA         & SPAN   & SPAGCL & \textbf{GCIR}                 \\ \hline
\multirow{4}*{Cora}&5\% & 75.25 (0.9) & 66.97(1.4) & 73.41(1.0) & 75.73(0.5) & 72.46(1.8) & 74.92(1.2) & 73.80(1.1) & 74.18(0.9) & 76.26 (0.8) & \textbf{78.78(1.0)} \\
                  &10\% & 70.27 (1.3) & 60.11(1.7) & 68.03(1.8) & 71.10(1.4) & 65.37(2.8) & 71.93(1.2) & 70.26(1.3) & 68.81(1.9) & 71.33 (0.5) & \textbf{75.93(0.9)} \\
                  &15\% & 65.94 (0.6) & 51.85(3.0) & 59.70(2.7) & 64.69(1.9) & 52.62(3.1) & 52.91(4.4) & 57.89(1.5) & 58.44(3.2) & 63.83 (0.7) & \textbf{71.89(1.6)} \\
                  &20\% & 56.54 (1.8) & 44.18(4.3) & 51.43(2.3) & 57.29(2.8) & 41.17(4.0) & 43.71(3.9) & 46.03(3.0) & 45.71(2.5) & 50.96 (1.0) & \textbf{68.81(2.4)} \\
\hline
\multirow{4}*{CiteSeer}
                       &5\%  & 69.60 (1.3) & 56.90(2.7) & 66.43(1.1) & 68.93(0.9) & 69.55(2.0) & 70.79(1.7) & 67.93(1.1) & 69.08(1.3) & \textbf{71.21 (1.2)} & 70.93(1.3) \\
                       &10\% & 65.40 (0.9) & 52.30(2.6) & 60.78(2.8) & 65.01(1.2) & 63.49(2.3) & 69.14(2.1) & 60.82(1.3) & 65.18(2.6) & \textbf{69.60 (1.4)} & 69.43(1.3) \\
                       &15\% & 59.42 (2.1) & 44.44(2.7) & 53.30(2.4) & 58.51(1.3) & 53.67(2.7) & 56.46(2.3) & 52.52(1.7) & 56.08(3.5) & 62.44 (0.9) & \textbf{65.84(1.2)} \\
                       &20\% & 51.78 (0.8) & 40.89(3.2) & 49.54(1.9) & 51.69(1.6) & 47.62(3.0) & 47.93(3.2) & 48.79(1.4) & 50.68(3.0) & 54.56 (0.6) & \textbf{62.65(1.5)} \\
\hline
\multirow{4}*{Cora-ML}
                      &5\%  & 68.59 (0.8) & 64.26(1.9) & 67.06(2.9) & 70.18(2.5) & 68.65(2.0) & 71.21(1.7) & 69.84(1.7) & 69.34(0.9) & 72.33 (1.1) & \textbf{77.90(0.8)} \\
                      &10\% & 49.78 (1.8) & 53.08(2.6) & 49.34(2.1) & 54.20(2.4) & 50.49(1.7) & 51.87(1.8) & 46.81(1.9) & 50.44(1.2) & 55.60 (0.9) & \textbf{71.57(1.7)} \\
                      &15\% & 39.19 (1.3) & 44.09(3.6) & 42.21(1.5) & 45.96(3.3) & 40.89(3.4) & 41.11(2.3) & 37.17(3.6) & 41.77(2.4) & 47.24 (1.3) & \textbf{64.56(1.8)} \\
                      &20\% & 29.45 (0.89) & 37.41(3.3) & 34.35(0.4) & 37.40(2.3) & 31.57(3.3) & 28.52(2.2) & 26.69(2.1) & 33.46(2.5) & 34.79 (1.3) & \textbf{56.43(2.5)} \\
\hline
\multirow{4}*{Photo}
                    &5\%  & 70.12 (1.7) & 66.50(2.8) & 62.24(3.4) & 73.14(1.3) & 70.87(1.1) & 69.85(2.4) & 76.45(2.1) & 72.46(2.3) & 70.07 (2.0) & \textbf{78.37(2.2)} \\
                    &10\% & 50.75 (2.3) & 52.62(3.9) & 47.76(3.7) & 58.63(3.3) & 51.39(1.5) & 53.92(1.8) & 60.57(2.0) & 63.46(0.4) & 56.98 (1.5) & \textbf{67.80(2.4)} \\
                    &15\% & 43.70 (1.) & 43.56(2.4) & 40.73(3.6) & 49.82(2.0) & 44.37(1.6) & 45.98(1.8) & 51.66(1.4) & 50.02(0.4) & 49.87 (0.9) & \textbf{59.19(1.5)} \\
                    &20\% & 27.35 (1.5) & 35.68(3.5) & 37.91(4.0) & 45.00(2.0) & 32.45(1.4) & 34.19(1.8) & 45.35(1.3) & 46.81(3.4) & 35.03 (1.3) & \textbf{54.68(2.9)} \\
\hline
\multirow{4}*{Computers}
                        &5\%  & 62.78 (1.5) & 63.12(1.3) & 63.91(1.9) & 69.16(0.6) & 66.22(1.0) & 67.51(1.7) & 68.02(0.9) & 64.06(1.4) & 64.85 (1.1) & \textbf{72.25(2.2)} \\
                        &10\% & 59.70 (2.5) & 61.33(2.1) & 57.06(1.9) & 64.58(0.4) & 58.07(2.0) & 57.04(1.5) & 58.90(2.1) & 54.81(1.3) & 57.68 (1.1) & \textbf{66.14(2.1)} \\
                        &15\% & 54.26 (0.7) & 52.59(3.4) & 51.40(2.0) & 54.99(2.6) & 52.58(2.1) & 43.65(2.6) & 52.90(2.6) & 45.63(3.2) & 54.82 (0.8) & \textbf{59.93(3.8)} \\
                        &20\% & 49.89 (1.3) & 37.61(2.3) & 41.54(4.9) & 44.98(3.5) & 37.21(3.9) & 34.71(1.7) & 38.81(3.6) & 37.39(2.1) & 42.87 (0.6) & \textbf{51.13(4.1)} \\
\hline
\multirow{4}*{WikiCS}
                     &5\%  & 41.58 (3.12) & 40.91(1.0) & 43.71(2.2) & 42.17(1.8) & 41.30(1.1) & 43.92(1.0) & 41.67(0.8) & 40.60(1.1) & 43.70 (1.3) & \textbf{55.93(0.8)} \\
                     &10\% & 32.27 (1.87) & 29.93(0.9) & 33.88(2.7) & 30.73(0.7) & 30.89(1.1) & 33.55(1.4) & 31.18(0.7) & 28.48(0.7) & 30.56 (1.2) & \textbf{49.30(0.6)} \\
                     &15\% & 25.68 (2.2) & 24.35(0.7) & 27.10(2.0) & 24.54(1.1) & 25.40(0.4) & 27.37(0.6) & 25.25(0.5) & 23.32(0.7) & 27.32 (1.5) & \textbf{42.96(1.1)} \\
                     &20\% & 23.83 (1.84) & 21.60(1.8) & 23.79(1.3) & 22.38(0.8) & 23.23(0.4) & 24.56(0.7) & 23.03(0.7) & 21.79(0.5) & 23.94 (0.9) & \textbf{36.47(1.0)} \\
\bottomrule[1.pt]
\end{tabular}}
\end{table*}

\begin{table*}[h]
\centering
    \caption{Performances on node classification against CLGA.}
	\label{tab-defend-CLGA}
	\resizebox{0.9\textwidth}{!}{%
\begin{tabular}{c|ccccccccccc}
\toprule[1.pt]
 Dataset & $\frac{B}{|E|}$ & PiGCL & BGRL       & DGI        & MVGRL               & GRACE       & ARIEL       & GCA         & SPAN   & SPAGCL     & \textbf{GCIR}                 \\ \hline
\multirow{4}*{Cora}
&5\%      & 79.72 (0.9) & 70.78(1.1) & 78.96(1.3) & \textbf{80.09(0.9)} & 78.95(0.9) & 78.47(0.7) & 78.55(0.7) & 78.96(0.7) & 77.46 (1.0) & 79.13(0.6)           \\
&10\%     & 78.11 (1.2) & 68.35(1.9) & 75.64(1.2) & 77.92(0.9)          & 76.79(1.4) & 76.31(0.9) & 77.15(0.8) & 77.22(1.1) & 75.93 (0.9) & \textbf{78.62(1.1))} \\
&15\%     & 77.01 (2.3) & 66.15(1.3) & 73.72(0.7) & 75.89(1.0)          & 75.38(1.1) & 74.75(0.9) & 74.71(1.2) & 76.11(1.2) & 74.65 (1.6) & \textbf{78.18(1.0)}  \\
&20\%     & 75.20 (1.8) & 64.44(1.8) & 71.94(1.7) & 74.64(0.8)          & 73.13(1.1) & 74.44(1.2) & 71.78(1.0) & 74.43(1.6) & 73.74 (1.4) & \textbf{76.82(1.2)} \\
\hline
\multirow{4}*{CiteSeer}
&5\%      & 68.95 (2.3) & 58.52(1.3) & 66.24(1.5) & 68.60(1.2) & 69.38(1.2) & 70.43(1.2) & 67.92(1.0)          & 70.38(1.1) & 
70.14 (1.4) & \textbf{71.18(1.2)} \\
&10\%     & 68.48 (1.8) & 56.27(1.3) & 65.62(1.5) & 67.65(1.0)          & 67.80(1.0) & 69.25(2.1) & 65.43(1.1)          & 68.74(1.3) & 
68.36 (0.9) & \textbf{70.11(0.9)} \\
&15\%     & 67.30 (1.5) & 54.09(2.0) & 63.75(1.6) & 66.29(0.8)          & 67.52(1.6) & 66.77(2.4) & 62.43(1.7)          & 67.71(1.3) & 
67.18 (1.0) & \textbf{69.58(1.5)} \\
&20\%     & 64.10 (1.8) & 51.24(1.9) & 62.55(1.5) & 64.24(1.3)          & 65.68(1.7) & 64.99(2.1) & 60.52(1.7)          & 66.62(1.4) & 
66.58 (1.1) & \textbf{67.90(1.1)} \\
\hline
\multirow{4}*{Cora-ML}
&5\%      & 78.38 (1.2) & 70.38(2.1) & 77.91(0.9) & 78.62(0.8) & 79.16(1.0) & 79.27(0.6) & 77.53(0.7)          & 79.36(1.0) & 
79.18 (1.3) & \textbf{80.02(0.8)} \\
&10\%     & 77.89 (0.7) & 68.75(1.7) & 76.40(1.1) & 76.29(0.9)          & 77.34(1.0) & 77.87(0.7) & 76.18(0.9)          & 77.96(0.8) & 
77.62 (1.3) & \textbf{78.68(0.9)} \\
&15\%     & 76.77 (1.6) & 67.38(1.1) & 75.46(1.4) & 74.61(0.6)          & 76.54(0.8) & 76.97(0.9) & 74.09(1.2)          & 77.10(0.6) & 
76.38 (0.8) & \textbf{77.34(0.9)} \\
&20\%     & 75.71 (1.2) & 65.09(2.0) & 73.42(1.7) & 73.40(0.8)          & 75.36(0.8) & 75.36(1.3) & 72.54(0.9)          & 75.99(0.5) & 
75.22 (1.5) & \textbf{76.72(0.8)} \\
\hline
\multirow{4}*{Photo}
&5\%      & 77.78 (1.2) & 80.79(1.4) & 75.80(2.8) & 84.86(0.7) & 81.91(0.8) & 82.28(1.0) & 82.25(1.1) & 84.18(0.8) & 83.62 (0.6) & \textbf{87.09(0.6)} \\
&10\%     & 75.33 (0.9) & 79.11(1.0) & 72.74(3.4) & 82.27(0.9) & 77.49(1.0) & 78.12(0.9) & 79.17(0.8) & 82.21(0.7) & 81.86 (0.9) & \textbf{85.04(0.5)} \\
&15\%     & 70.97 (1.8) & 77.83(0.9) & 68.99(0.5) & 80.39(0.7) & 75.28(1.0) & 76.98(0.6) & 76.68(0.8) & 80.78(0.5) & 80.83 (1.1) & \textbf{82.11(1.1)} \\
&20\%     & 67.15 (1.6) & 76.97(0.8) & 68.59(2.0) & 78.88(0.8) & 71.85(1.2) & 74.92(1.3) & 73.10(0.7) & 78.90(0.9) & 77.66 (1.1) & \textbf{79.21(2.2)} \\
\bottomrule[1.pt]
\end{tabular}}
\end{table*}

\begin{figure*}[h]
	\begin{center}
        \subfloat[CLGA-0.05]{\includegraphics[width=0.24\textwidth,height=3.cm]{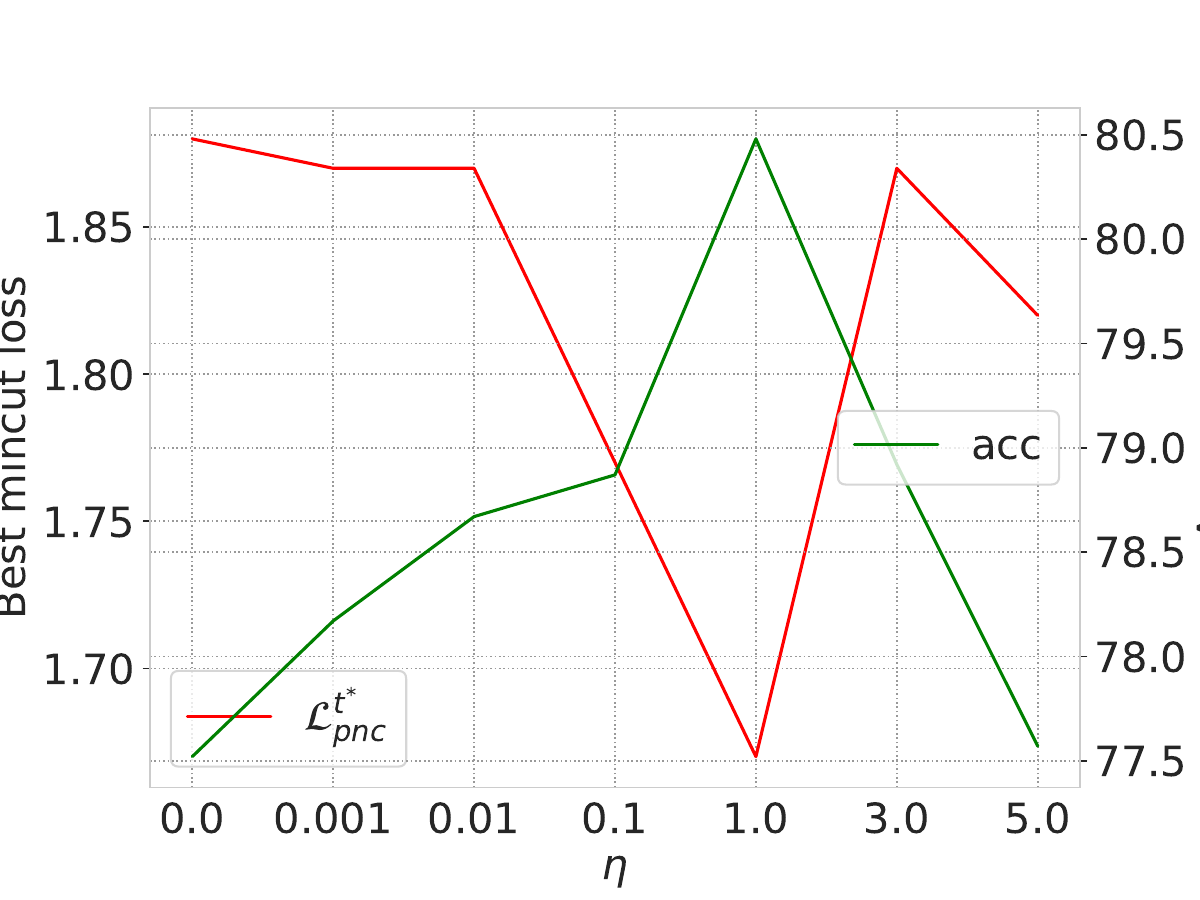}}
        \hfill
		\subfloat[CLGA-0.1]{\includegraphics[width=0.24\textwidth,height=3.cm]{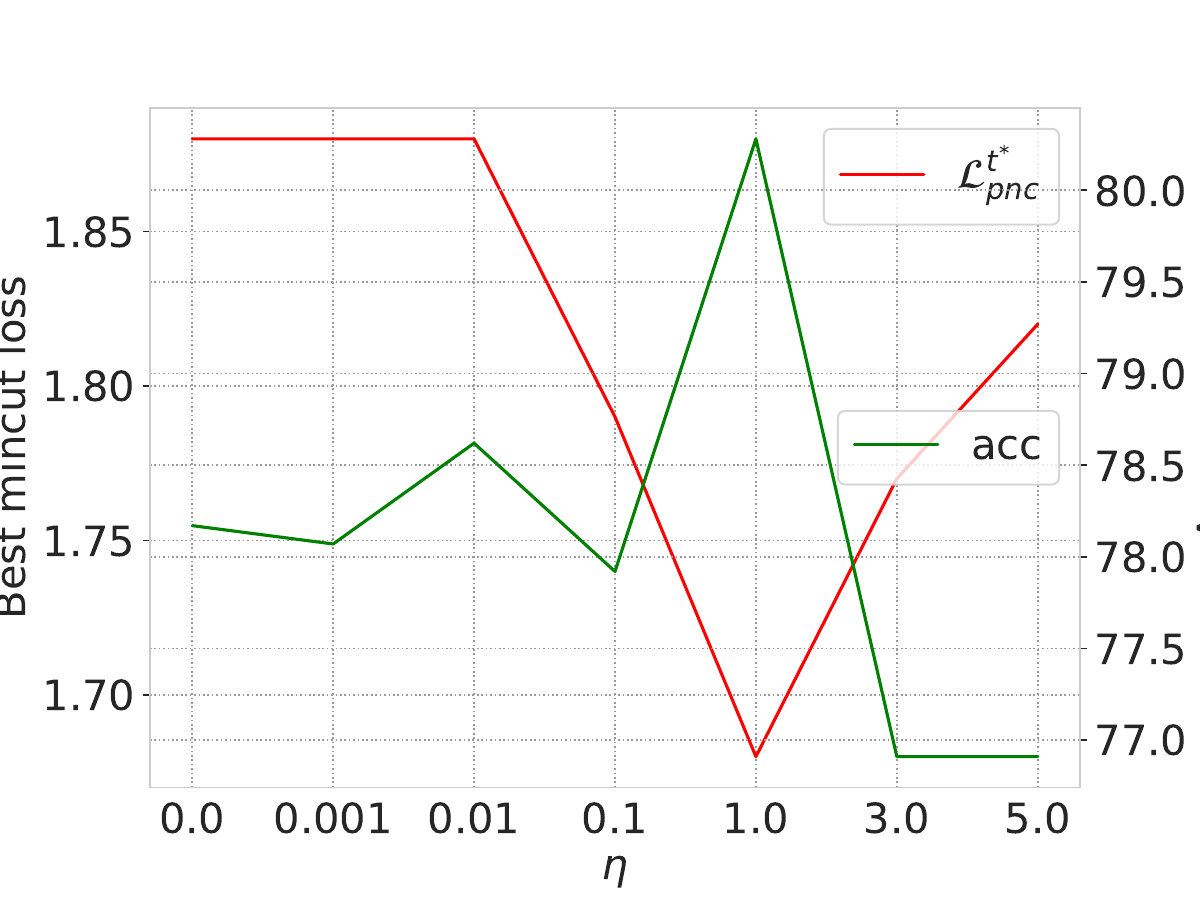}}
        \hfill
        \subfloat[CLGA-0.15]{\includegraphics[width=0.24\textwidth,height=3.cm]{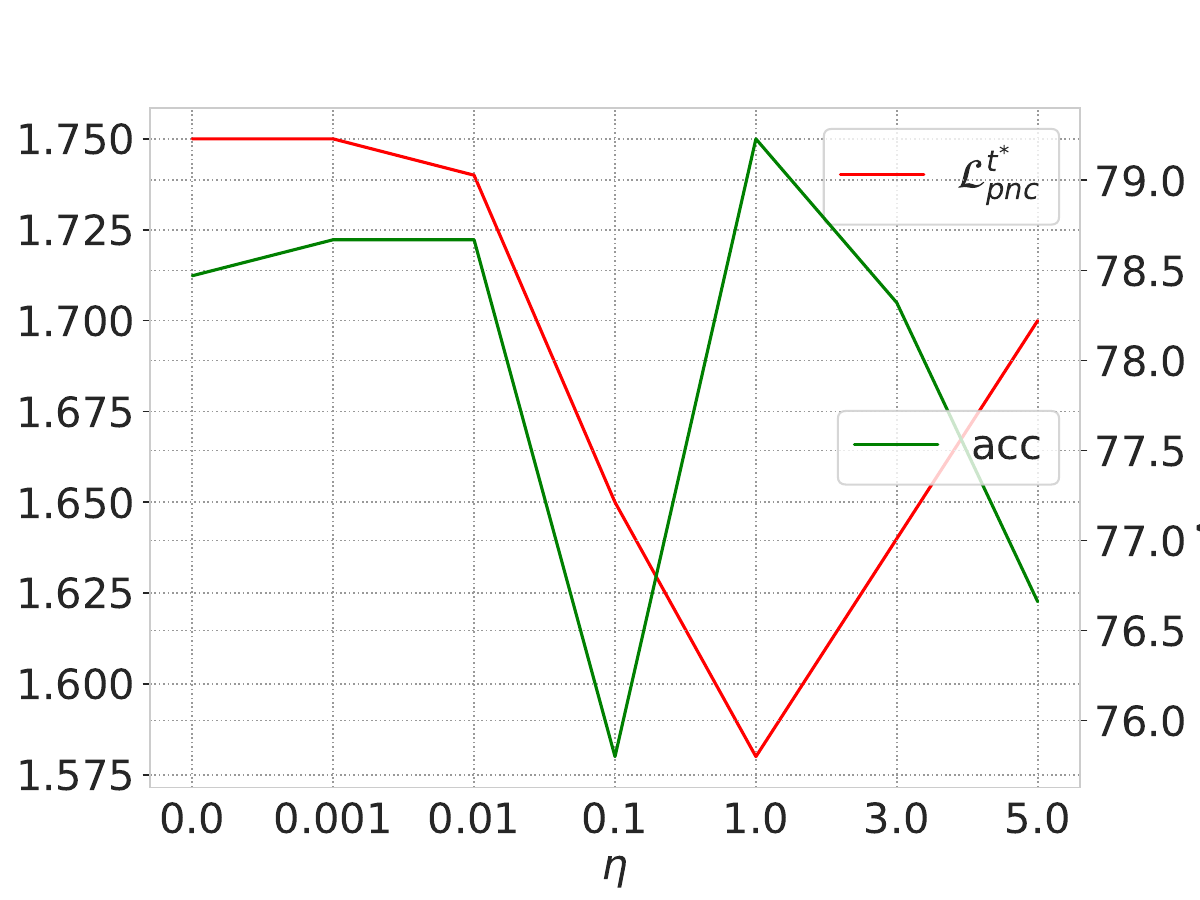}}
        \hfill
        \subfloat[CLGA-0.2]{\includegraphics[width=0.24\textwidth,height=3.cm]{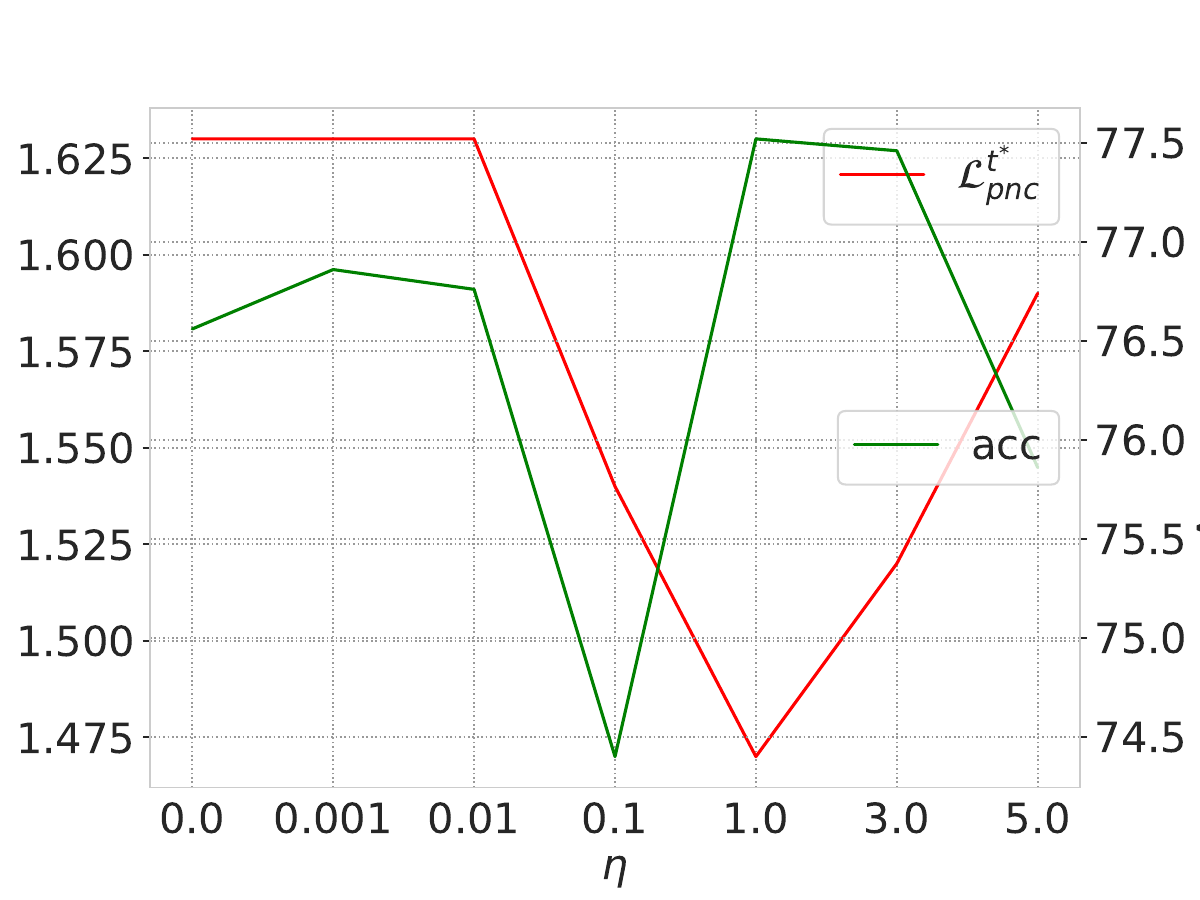}}
        \hfill
        \subfloat[Mettack-0.05]{\includegraphics[width=0.24\textwidth,height=3.cm]{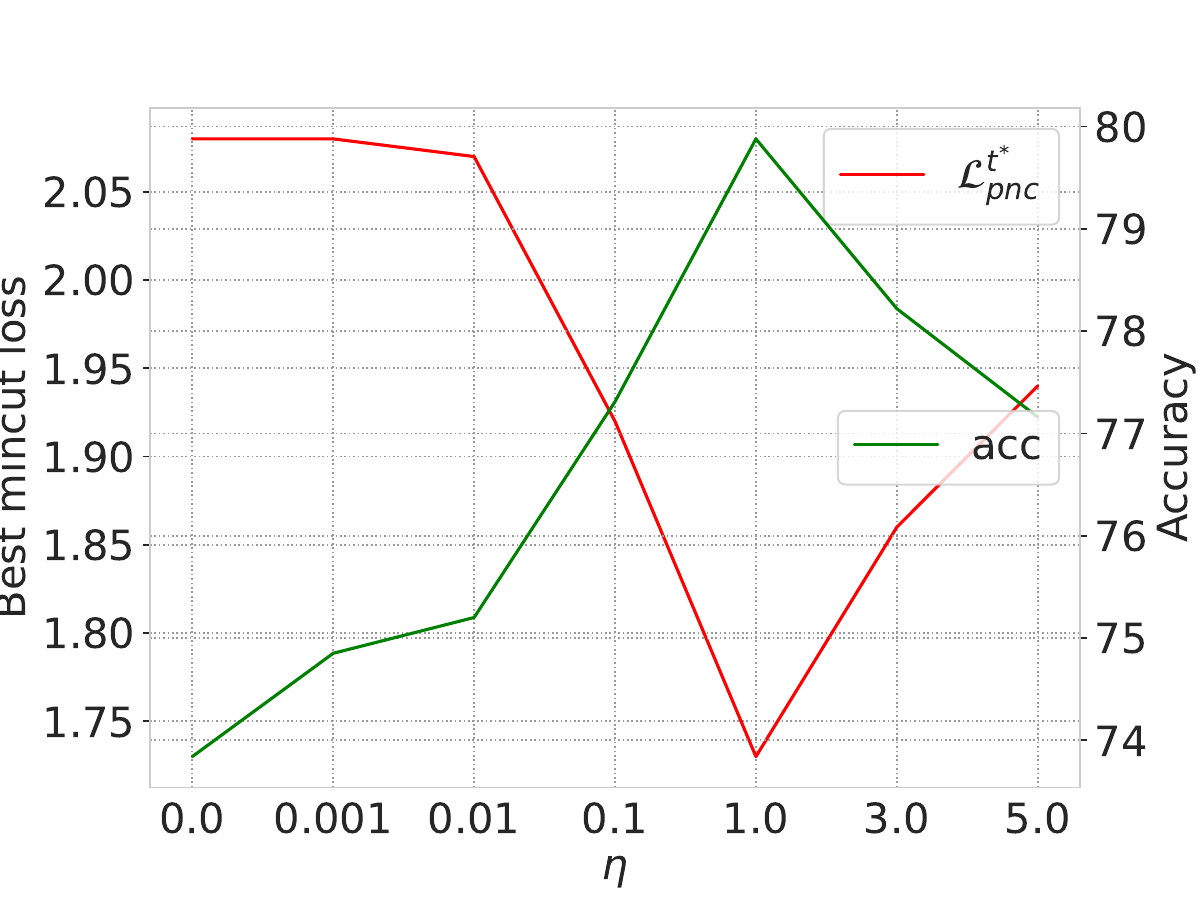}}
        \hfill
        \subfloat[Mettack-0.1]{\includegraphics[width=0.24\textwidth,height=3.cm]{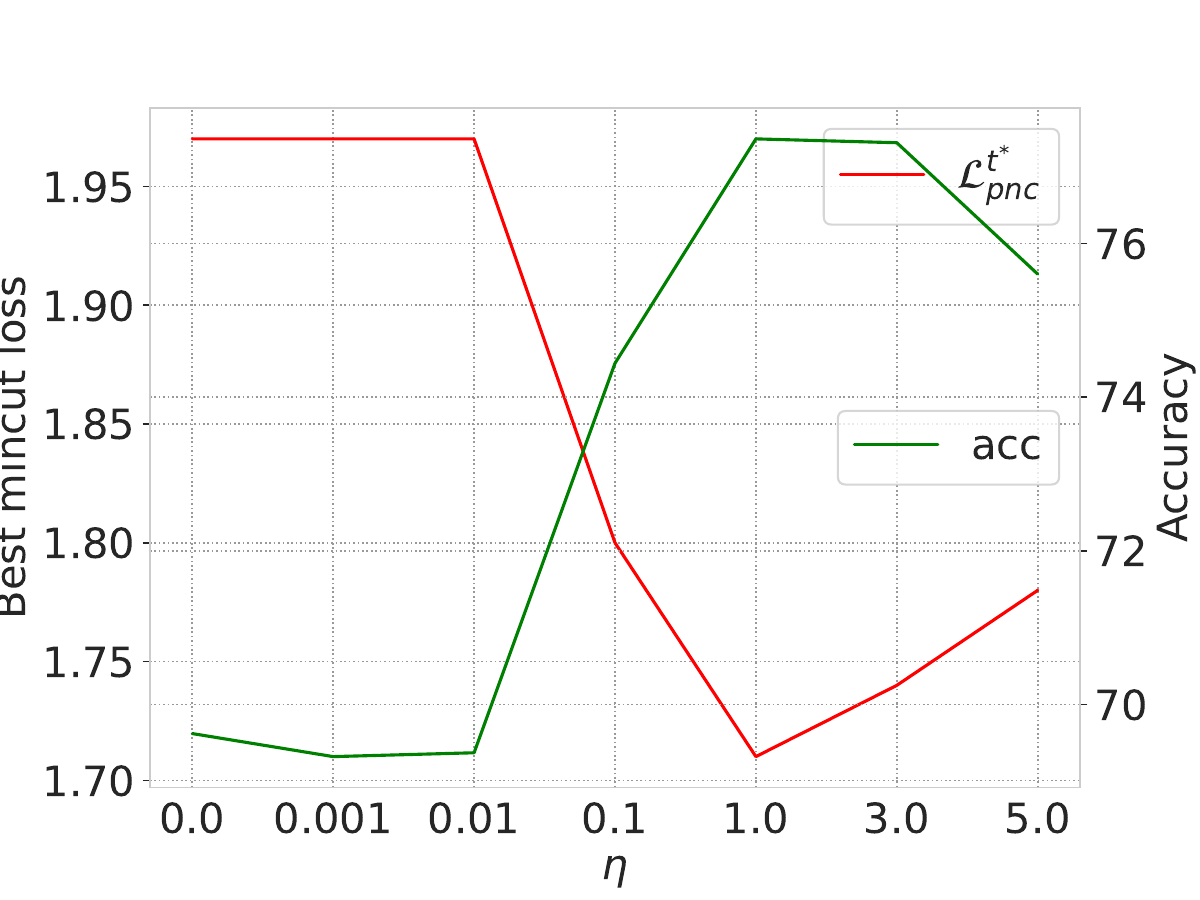}}
        \hfill
        \subfloat[Mettack-0.15]{\includegraphics[width=0.24\textwidth,height=3.cm]{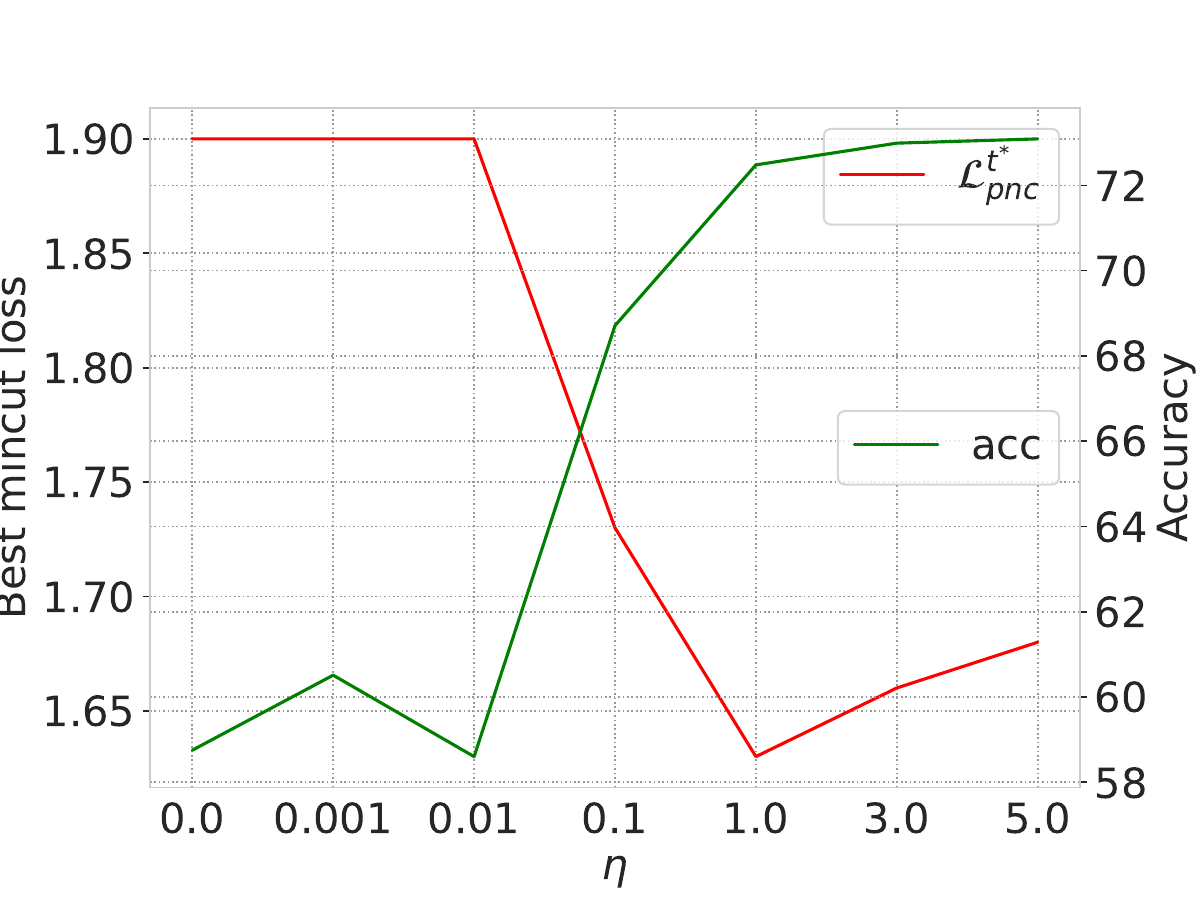}}
        \hfill
        \subfloat[Mettack-0.2]{\includegraphics[width=0.24\textwidth,height=3.cm]{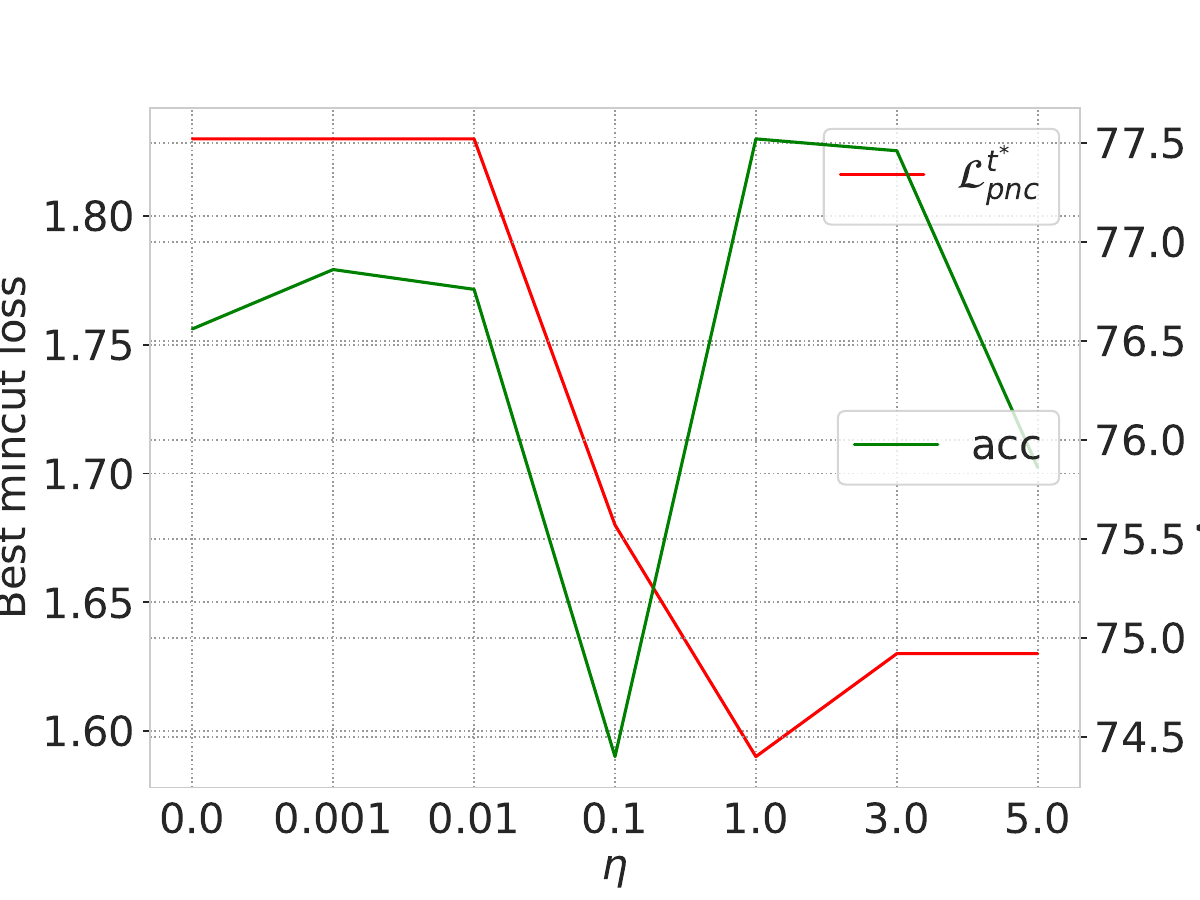}}
	\end{center}
	\caption{Sensitivity analysis on $\eta$.}
	\label{fig-Sensitivity}
\end{figure*}

\subsection{Baselines}
\label{Sec-Baselines}
To thoroughly validate the robustness of our proposed method, we have conducted a comparative analysis with several state-of-the-art graph contrastive learning models: BGRL~\cite{BGRL}, DGI~\cite{DGI}, MVGRL~\cite{MVGRL}, GRACE~\cite{GRACE}, GCA~\cite{GCA}, ARIEL~\cite{ARIEL}, and SPAN~\cite{SPAN}. Among them, ARIEL and SPAN are known for their robustness in GCL. Meanwhile, it is worth noting that the defense methods proposed in \cite{GRV, LocalGlobal} require access to the \textit{clean} graph, which is not available in a more practical setting such as ours, ARIEL and SPAN. We list the brief introduction of these baselines as follows:
\begin{itemize}
    \item \textbf{BGRL}~\cite{BGRL} It introduces the Bootstrapped Graph Latents technique to improve the scalability of graph contrastive learning effectively. It contrasts the two embeddings from distinct encoders rather than augmentation views, in which one encoder is the exponential moving average of the other. 
    \item \textbf{DGI}~\cite{DGI} It relies on maximizing the mutual information between patch representations and corresponding high-level summaries of graphs, such as the readout mapping of the node embeddings. 
    \item \textbf{MVGRL}~\cite{MVGRL} It learns node and graph level representations by contrasting low-dimensional embeddings from two augmented views of graphs, including first-order neighbors and a graph diffusion. 
    \item \textbf{GRACE}~\cite{GRACE} It is an unsupervised graph representation learning that generates two augmentation views by corruption and learns node representations by maximizing the agreement of node representations in these two views.
    \item \textbf{ARIEL}~\cite{ARIEL} it leverages the adversarial training strategy to construct the adversarial view with information regularization to achieve adversarial robustness. 
    \item \textbf{GCA}~\cite{GCA} It is a graph contrastive learning framework that employs more complex adaptive augmentation schemes such as degree centrality, eigenvector centrality, and PageRank centrality.
    \item \textbf{SPAN}~\cite{SPAN} It introduces two augmentation views to minimize and maximize the spectrum of the augmented graphs to preserve the spectrum invariance of the GCL.
    \item \textbf{PiGCL}~\cite{PiGCL} It is a graph contrastive learning framework that dynamically captures the implicit conflicts from the negative pairs during training by detecting the gradient of representation similarities to eliminate the conflict phenomenon caused by the InfoNCE loss.
    \item \textbf{SPAGCL}~\cite{SPAGCL} It is a graph contrastive learning framework that combines the node similarity-preserving view and an adversarial view to improve the node similarity during the adversarial training.
\end{itemize}

%To ensure a rigorous comparison, we have used the same hyperparameter settings as reported in the corresponding literature, where all the baselines are fine-tuned based on the performance of the validation set. 

\subsection{Robustness on Node Classification}
\label{Sec-Robust-NC}
The objective of this paper is to examine the adversarial robustness of our proposed model against two common graph structural attacks: Mettack~\cite{Mettack} (a graph structural attack with GNN as its surrogate model) and CLGA~\cite{CLGA} (a graph structural attack with GCL as its surrogate model). 
\subsubsection{Against Mettack}
Tab.~\ref{tab-defend-Mettack} displays the accuracies of various GCL models for the semi-supervised node classification under different attacking powers. In particular, the attacking power is defined as the percentage $\frac{B}{|E|}$, where $B$ is the attack budget (number of modified edges) and $|E|$ is the total edge number of the clean graph. To cover most attacking scenarios, we have chosen the attacking power from $\{5\%, 10\%, 15\%, 20\%\}$, which is consistent with the prior literature~\cite{Mettack, CLGA}. 

The results in Tab.~\ref{tab-defend-Mettack} indicate that Mettack affects all GCL methods, although it was originally designed to attack vanilla GNNs. This is because the GCL models also rely on the GNN encoder to convert the graph data into low-dimensional embeddings, and thus, Mettack can indirectly compromise the quality of the GCL's node embeddings.
%by decreasing the homophily degree of the clean graph.  

The second observation is that our proposed model outperforms the other baselines across various levels of attacking power. Moreover, the performance gap between our model and other baselines widens as the attacking power grows from $5\%$ to $20\%$. For example, on the Cora dataset, the performance gaps between our model with GRACE~\cite{GRACE}, ARIEL~\cite{ARIEL}, GCA~\cite{GCA}, SPAN~\cite{SPAN}, BGRL~\cite{BGRL}, DGI~\cite{DGI} and MVGRL~\cite{MVGRL} are $8.02\%$, $4.90\%$, $6.32\%$, $5.84\%$, $14.99\%$, $6.82\%$ and $3.87\%$ when the attacking power is $5\%$. The corresponding margins under $20\%$ attacking power are $40.17\%$, $36.47\%$, $33.11\%$, $33.57\%$, $35.79\%$, $25.26\%$, $16.74\%$. This can be attributed to the fact that as the attacking power grows, more malicious links are inserted. However, the sanitation view in our proposed model can effectively eliminate a greater number of malicious edges from the heavily poisoned graph, thereby improving the quality of node embeddings and making them more distinguishable across different classes. 
%In addition to that, it is observed that in certain scenarios, our model cannot surpass other strong baselines on clean graphs. For instance, on the CiteSeer dataset, the node classification accuracy of our model is slightly lower than that of GCA and MVGRL by $0.57\%$ and $0.55\%$. This phenomenon is common in the field of adversarial machine learning and can be attributed to the trade-off between expressive power and the robustness of the model~\cite{robustness_tradeoff, zhang2019theoretically}. 

\subsubsection{Against CLGA}
Tab.~\ref{tab-defend-CLGA} presents the robustness results under the unsupervised attack CLGA. Similar to its performance under Mettack, our model outperforms other baselines across various levels of attacking power. However, since CLGA is an unsupervised attack, the degree of decreasing node classification accuracy for GCL models is not as significant as in Mettack, which specifically targets the test set. 
We note that MVGRL performs well against CLGA and even outperforms our method when the attacking power is equal to $5\%$ on the Cora dataset. This is because  MVGRL does not use the InfoNCE loss for training, which is the attack objective of CLGA. Nevertheless, our model outperforms MVGRL on the Cora dataset when the attacking power is larger than or equal to $10\%$. Overall, the proposed learnable sanitation view effectively promotes the robustness of the vanilla GCL framework by automatically detecting and pruning potential malicious edges in the poisoned graph, as demonstrated by the robust performance results shown in Tab.~\ref{tab-defend-Mettack} and Tab.~\ref{tab-defend-CLGA}.

\begin{figure}[h]
	\begin{center}
		\subfloat[Mettack]{\includegraphics[width=0.24\textwidth,height=3.cm]{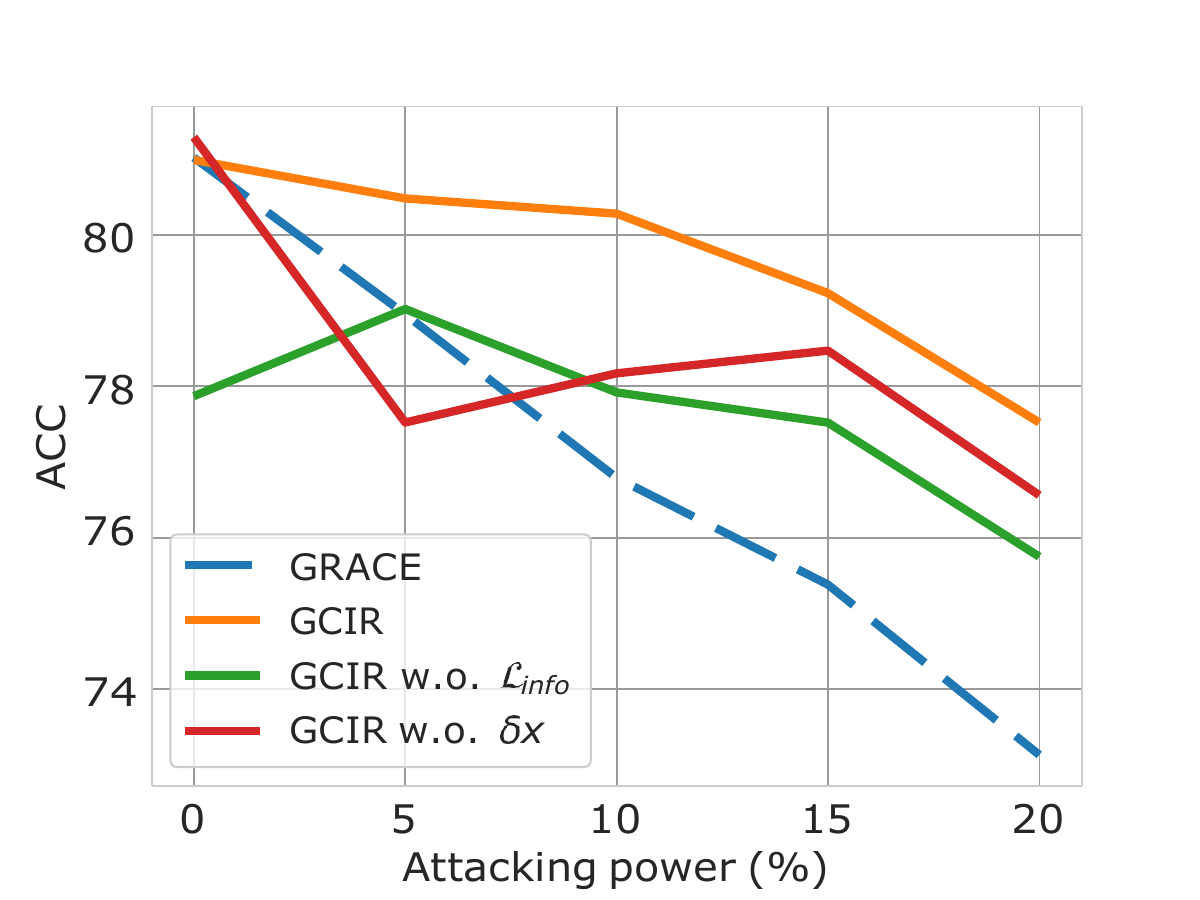}}
        \hfill		
        \subfloat[CLGA]{\includegraphics[width=0.24\textwidth,height=3.cm]{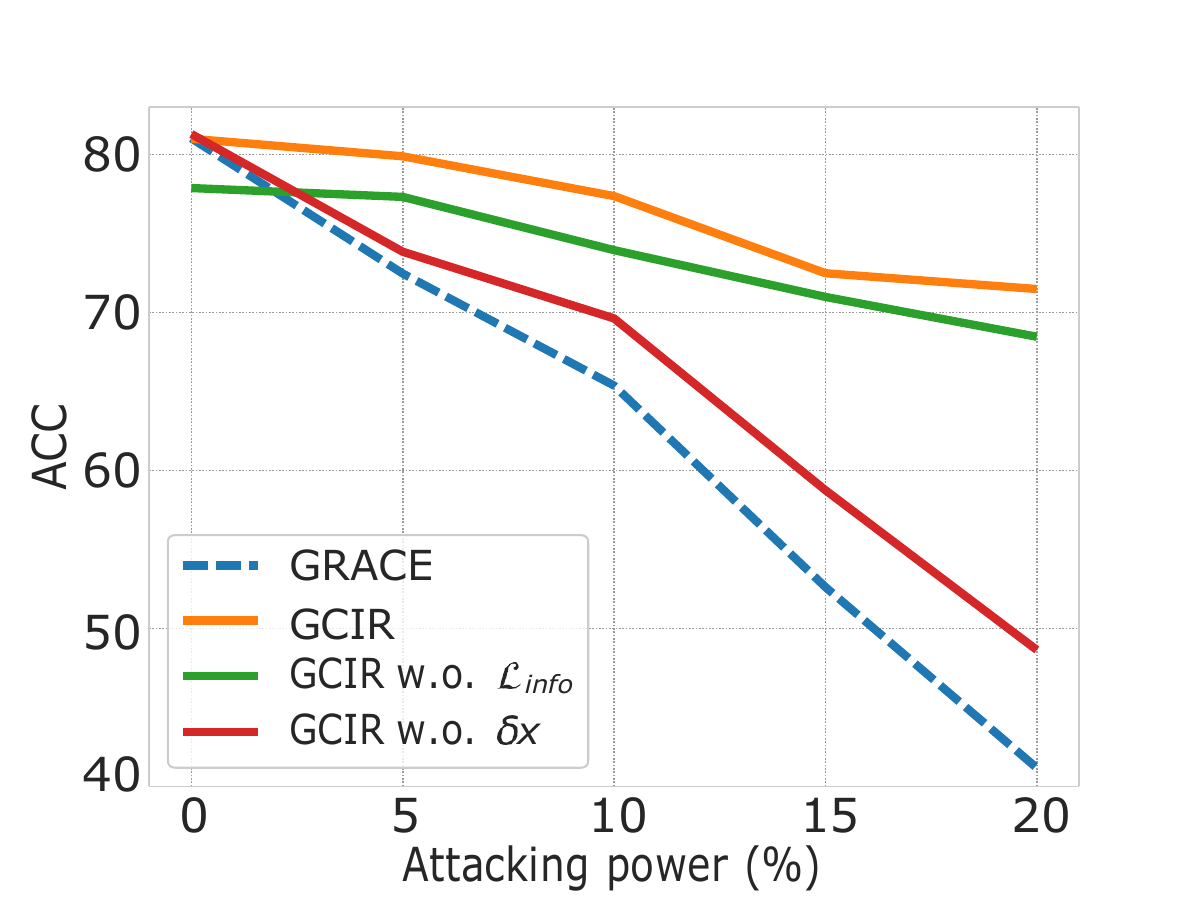}}
	\end{center}
	\caption{Ablation study of two model variants.}
	\label{fig-Ablation}
\end{figure}

\begin{figure}[h]
	\begin{center}
		\subfloat[CLGA]{\includegraphics[width=0.24\textwidth,height=3.cm]{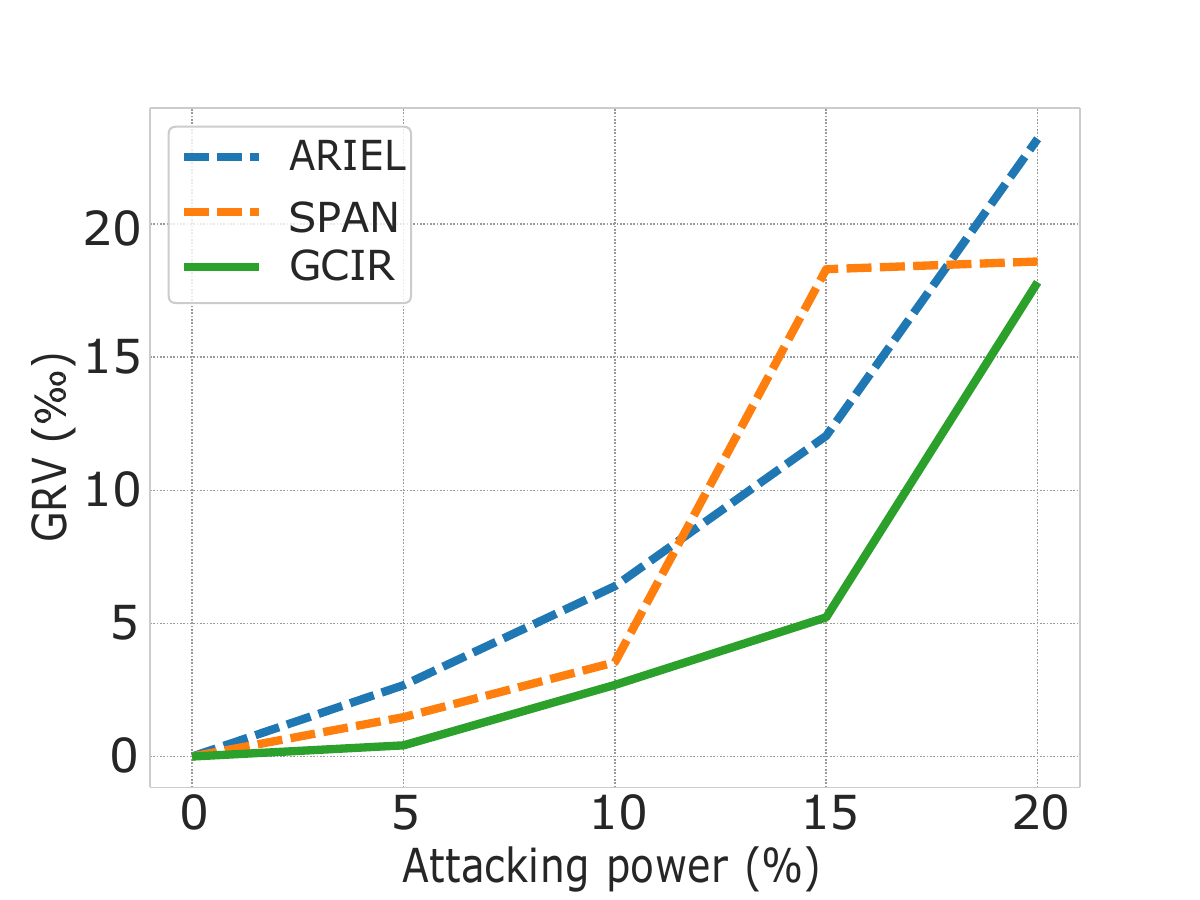}}
        \hfill		
        \subfloat[Mettack]{\includegraphics[width=0.24\textwidth,height=3.cm]{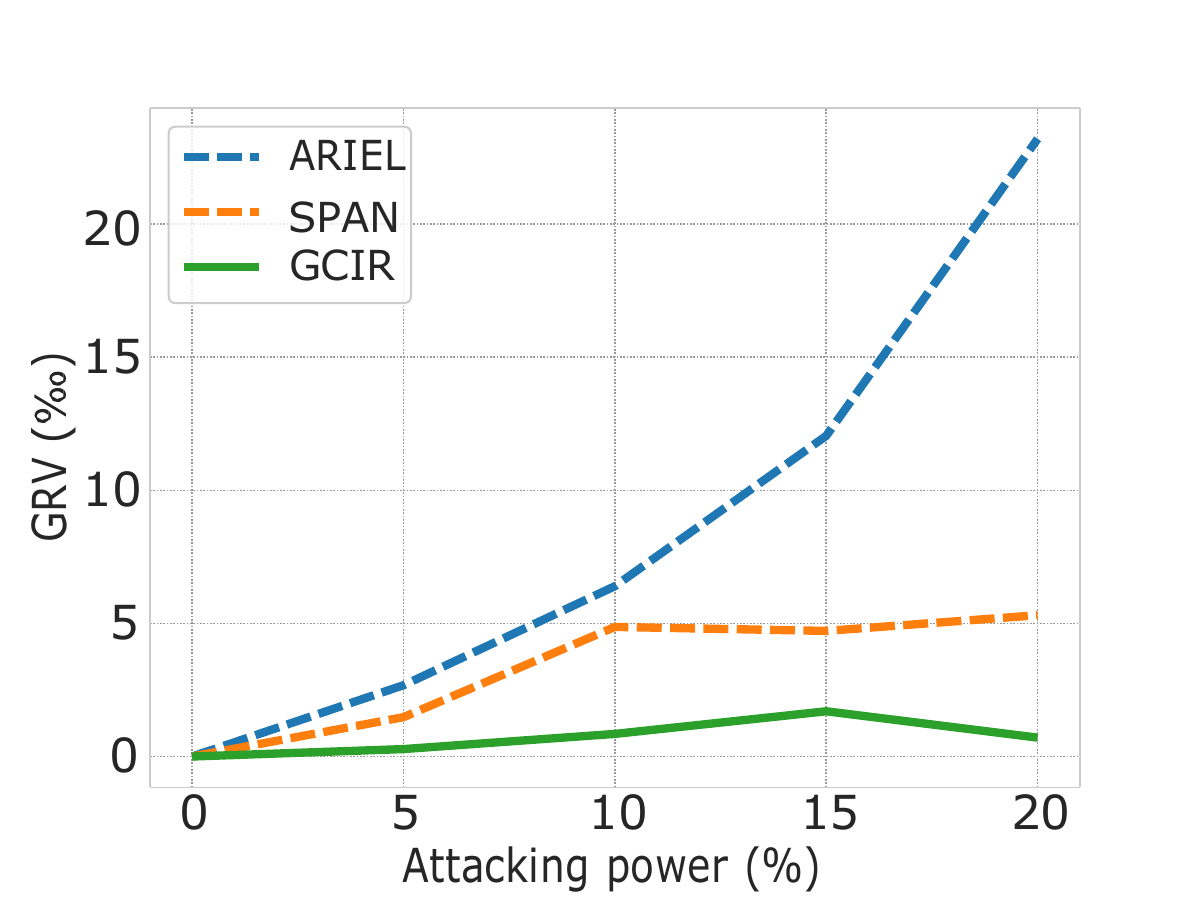}}
	\end{center}
	\caption{Evaluation of the adversarial robustness of robust GCLs based on GRV for different attacking scenarios.}
	\label{fig-GRV}
\end{figure}

\subsection{Ablation Study}
In this section, we present the ablation study to verify the effectiveness of two vital components of the loss function defined in Eqn.~\ref{eqn-san-obj-2}, corresponding to restoring the mutual information between the graph and its representations and the graph homophily respectively. In Particular, we introduce two model variants as follows:
\begin{itemize}
    \item \textbf{GCIR} w.o. $\mathcal{L}_{info}$: we remove the InfoNCE loss in Eqn.~\ref{eqn-san-obj-2} and first train the sanitation view to optimize $\delta_{x}$. Then, we train the GCL model with the pre-trained sanitation view. 
    \item \textbf{GCIR} w.o. $\delta_{x}$: we remove the $\delta_{x}$ term in Eqn.~\ref{eqn-san-obj-2} and only optimize the InfoNCE loss in \textbf{GCIR}.
\end{itemize}
The results are presented in Fig.~\ref{fig-Ablation}. Firstly, comparing the robust performance of \textbf{GCIR} w.o. $\mathcal{L}_{info}$, \textbf{GCIR} w.o. $\delta_{x}$ and GRACE, we find out that both optimizing the InfoNCE loss and the graph homophily $\delta_{x}$ can enhance the robustness of the GCL model since in most cases the two model variants outperform GRACE. Secondly, the comparison of \textbf{GCIR} w.o. $\mathcal{L}_{info}$ and \textbf{GCIR} w.o. $\delta_{x}$ under CLGA reveals that restoring the diminished mutual information will be more effective than the graph homophily. Thirdly,  the comparison of \textbf{GCIR} w.o. $\mathcal{L}_{info}$ and \textbf{GCIR} w.o. $\delta_{x}$ under Mettack indicates that restoring the graph homophily will be more effective than the mutual information. The reason is probably that Mettack is such attacks that degenerate the GNN's performance, which does not directly rely on information theory.  

\subsection{Sensitivity Analysis}
In this section, we discuss the tuning strategy of the vital hyperparameter $\eta$ in Eqn.~\ref{eqn-san-obj-2}. To mimic the real environment, it is assumed that the defender cannot acquire the attacking scenarios (such as the attack type, attack degree, and victim nodes or edges and etc.), and the label information of the given graph. That is, the defender can only utilize the attribute matrix, the suspected perturbed adjacency matrix, and the node embeddings of GCL to determine the best choice of $\eta$. To tackle this issue, as previously mentioned in Sec.~\ref{sec-unsupervised-tuning}, we craft a pseudo normalized cut loss $\mathcal{L}_{pnc}^{*}$ to tune the vital hyperparameter $\eta$.
To be specific, we report the best pseudo normalized cut loss $\mathcal{L}_{pnc}^{t^{*}}$ and the node classification accuracy with attack power equal to $10\%$ for CLGA and Mettack as exemplar and present the results in Fig.~\ref{fig-Sensitivity}. Moreover, we tune the hyperparameter $\eta$ ranged in $[0,0.0001,0.001,0.01,0.1,1,3,5].$ The experiments demonstrate that too large and too small $\eta$ may lead to sub-optimal performance on \textbf{GCIR} under different attacking scenarios. Alternatively, for all the cases listed in Fig.~\ref{fig-Sensitivity} $\eta=1$ can achieve the minimum $\mathcal{L}_{pnc}^{t^{*}}$ (coincide with the best node classification accuracy). %\textcolor{blue}{It is noticed that in some cases when $\eta\in\{0.001, 0.01\}$ the $\mathcal{L}_{pnc}^{t^{*}}$ remains unchanged reveal that it is insensitive to the extremely low value of the vital hyperparameter $\eta$. Alternatively, when $\eta$ increases to more than $0.1$ we could observe significant changes in the traceplot.}  

\begin{figure}[h]
	\begin{center}
		\subfloat[GCIR-CLGA]{\includegraphics[width=0.24\textwidth,height=3.cm]{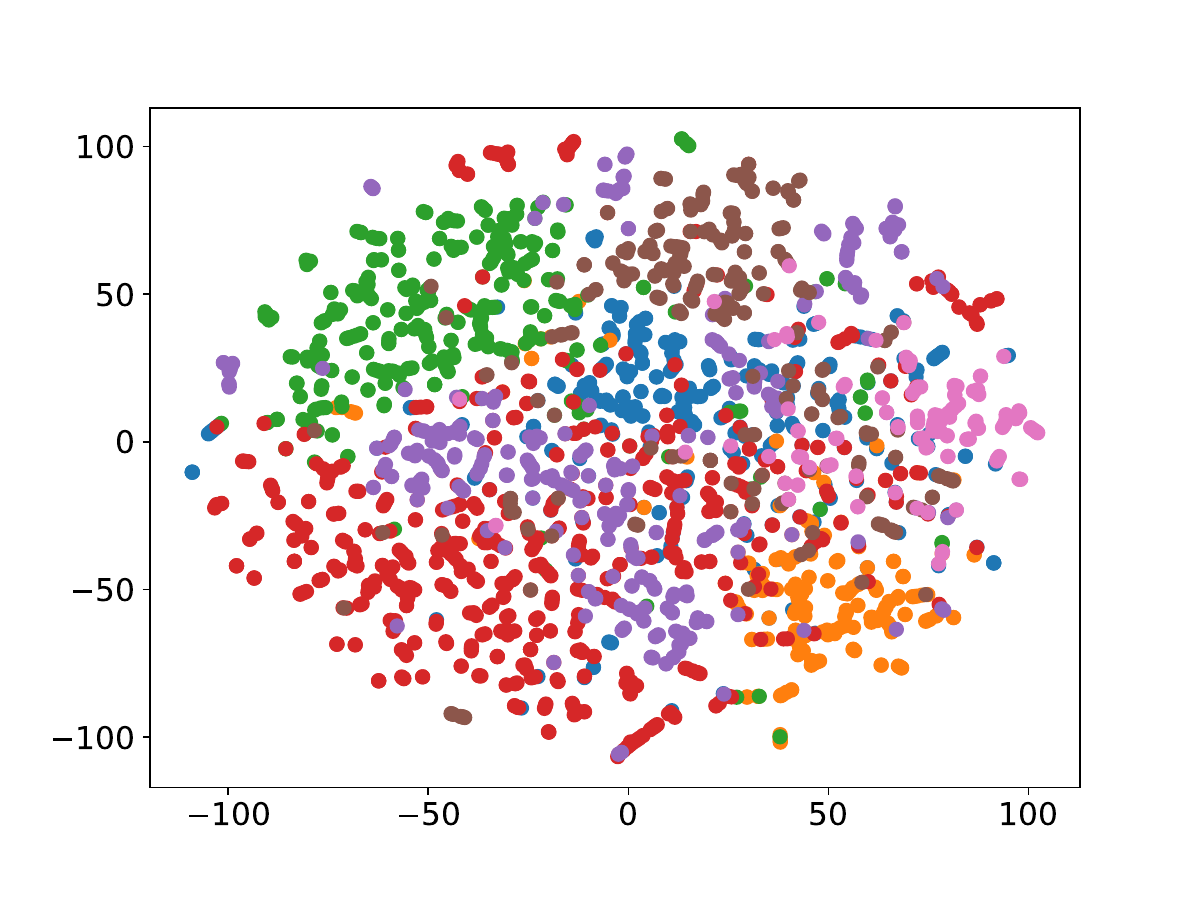}}
        \hfill		
        \subfloat[GCIR-Mettack]{\includegraphics[width=0.24\textwidth,height=3.cm]{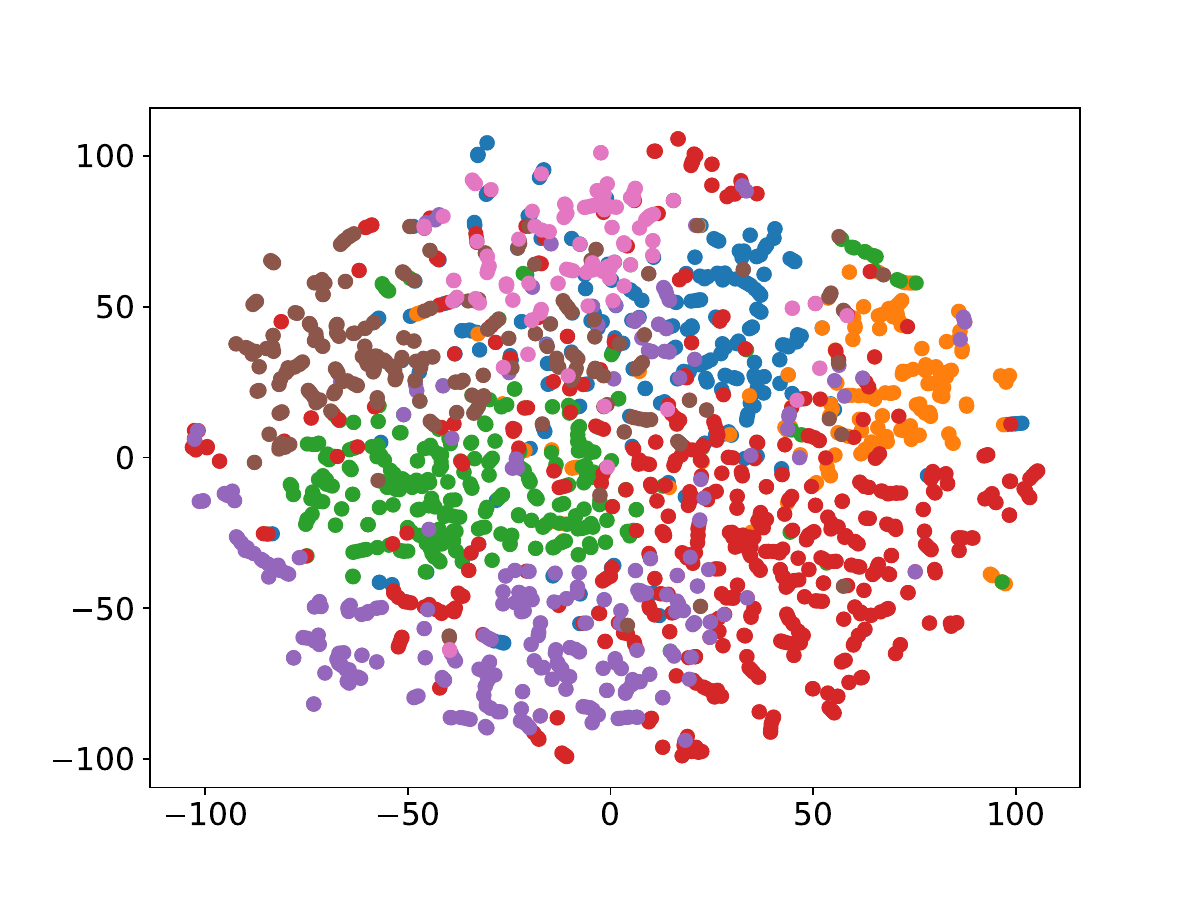}}
        \hfill
        \subfloat[SPAN-CLGA]{\includegraphics[width=0.24\textwidth,height=3.cm]{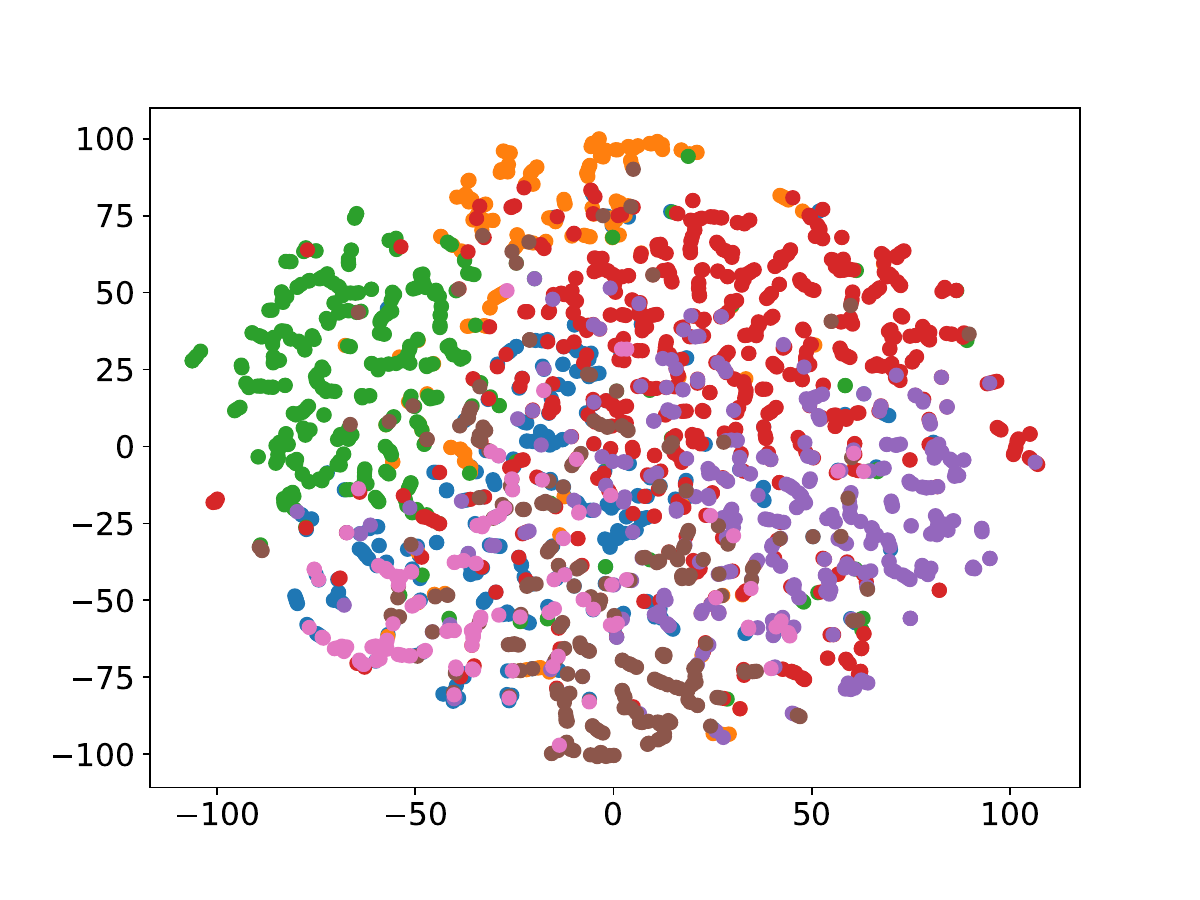}}
        \hfill		
        \subfloat[SPAN-Mettack]{\includegraphics[width=0.24\textwidth,height=3.cm]{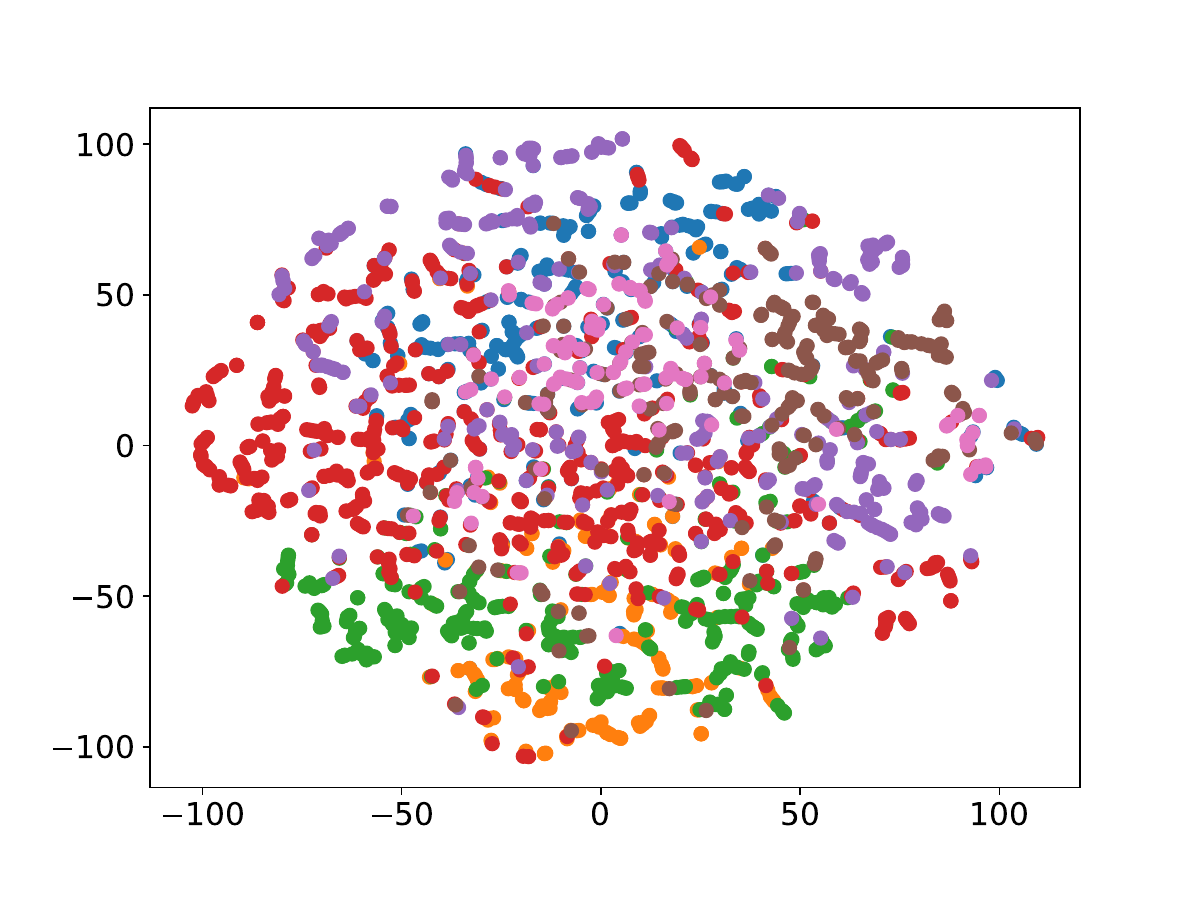}}
        \hfill
        \subfloat[ARIEL-CLGA]{\includegraphics[width=0.24\textwidth,height=3.cm]{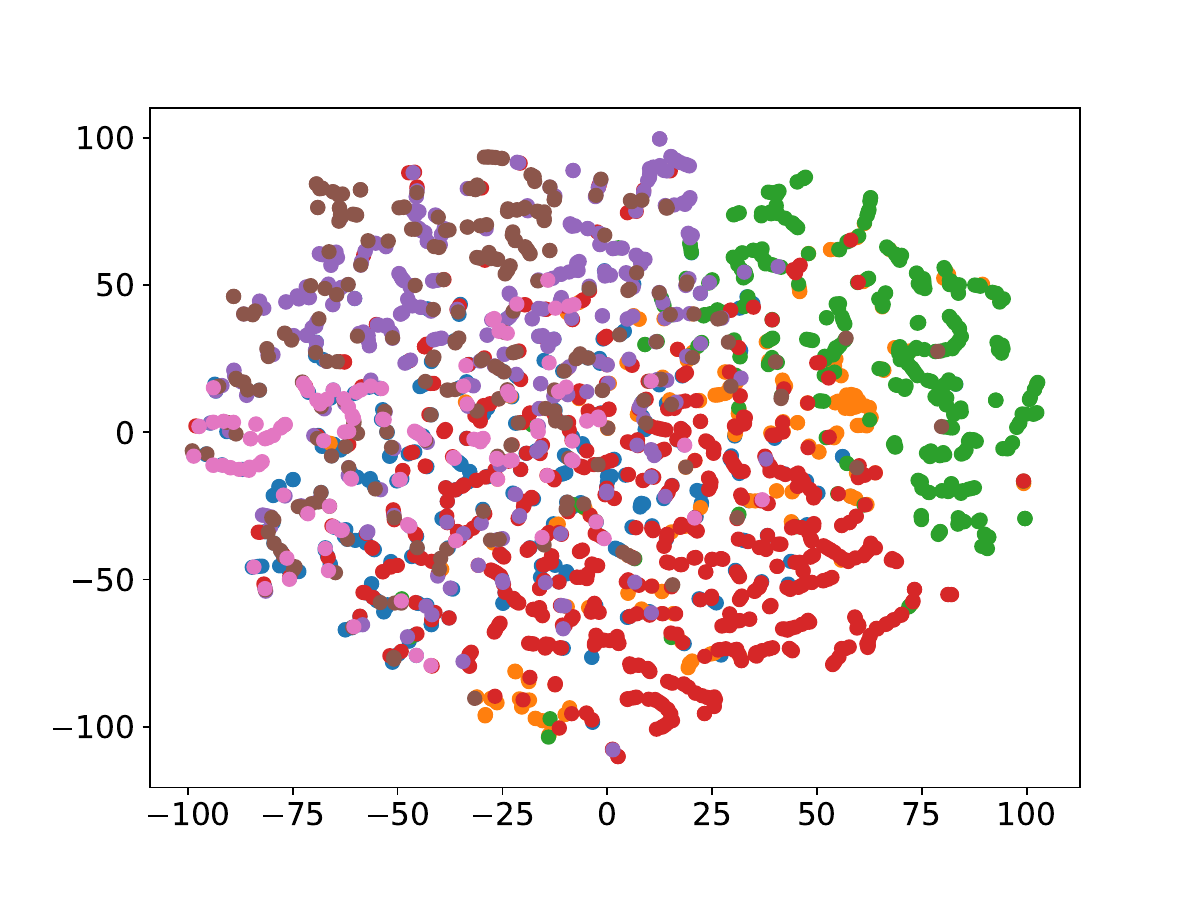}}
        \hfill		
        \subfloat[ARIEL-Mettack]{\includegraphics[width=0.24\textwidth,height=3.cm]{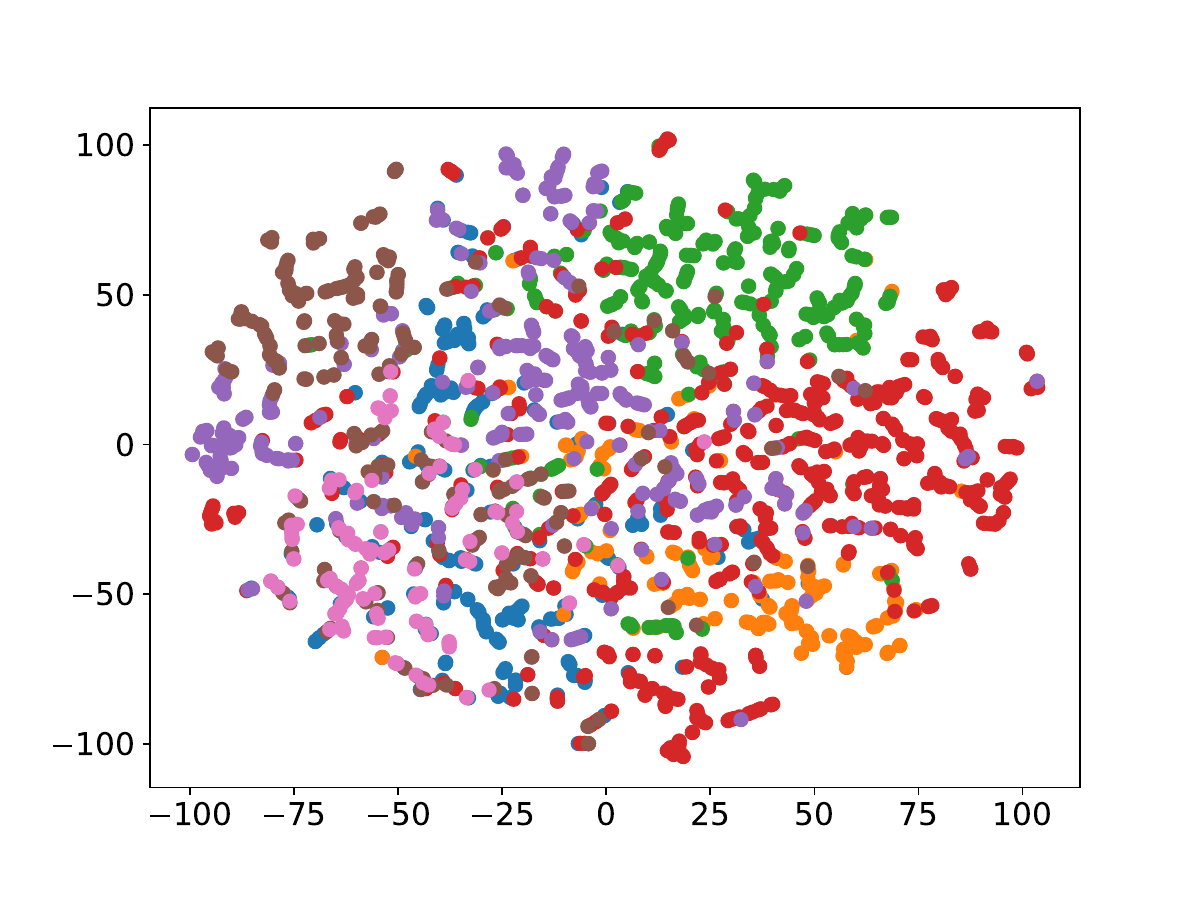}}
	\end{center}
	\caption{Scatterplots for GCIR, SPAN and ARIEL.}
	\label{fig-visual-embs}
\end{figure}

\subsection{Node Embeddings Visualization}
\label{sec-node-embs-visualization}
We provide the node embeddings visualization for \textbf{GCIR}, SPAN, and ARIEL (robust GCLs) to qualitatively evaluate the quality of node embeddings under different attacking scenarios (CLGA and Mettack with attacking power equal to $20\%$) in Fig.~\ref{fig-visual-embs}. The visualization is depicted on the transformed node embeddings generated by \textbf{GCIR}, SPAN, and ARIEL via t-SNE~\cite{tSNE}. Clearly, the node embeddings for \textbf{GCIR} are more discriminative than SPAN and ARIEL and the clusters of \textbf{GCIR} are more cohesive than the others. 

\subsection{Robustness on Graph Representation Vulnerability}
In this section, we introduce the concept of graph representation vulnerability and its relationship with the evaluation of the GCL's robustness. We then test the robustness of the robust GCLs based on the defined metric from an information-theoretical perspective. 

\begin{definition}[GRV~\cite{GRV, LocalGlobal}]
\label{lemma-GRV}
    The graph representation vulnerability (GRV) quantifies the discrimination between the mutual information of the graph and its node representations based on clean graph and poisoned graph:
    \begin{equation}
        \begin{split}
            GRV(G,G^p)=I(G;f_{\theta}(G))-I(G^p;f_{\theta}(G^p)).
        \end{split}
    \end{equation}
\end{definition}
Basically, GRV can be utilized as a metric to quantify the adversarial robustness of an embedding-generating model $f_\theta$. Intuitively, when GRV is smaller, the embeddings generated from the poisoned graph are closer to those generated from the clean graph; that is, the model $f_\theta$ is more robust.  
%}
\begin{definition}[Defense of GCL]
\label{def-GCL-defense}
Given the poisoned graph generated by the graph attacker $G^{p}=\{\mathbf{X}, \mathbf{A}^{p}\}$, the defender aims to train a robust model $f_{\theta^{*}}$
%the node embeddings $\mathbf{Z}^{*}=f_{\theta^{*}}(\mathbf{X}, t^{*}(\mathbf{A}^{p}))$ which 
that can minimize the GRV defined in Definition~\ref{lemma-GRV}:
\begin{equation}
    \label{eqn-infomax}
    \begin{split}
        %\mathbf{Z}^{*}=\arg\min I(\mathbf{X},f_{\theta}(G))-I(\mathbf{X},\mathbf{Z}^{*}(G^p)),
        f^{*}=\arg\min_{f\in\mathcal{F}} I(G;f_{\theta}(G))-I(G^p;f_{\theta}(G^p)).
    \end{split}
\end{equation}

where $I(\cdot,\cdot)$ represents the mutual information, $f_{\theta}(\cdot)$ represents the GNN encoder parametrized by weights $\theta$. 
%$t(\cdot)$ denotes the learnable augmentation view. 
\end{definition}
We note that the clean graph $G$ required to compute GRV is not available during the learning process. Thus, GRV is solely used as a robustness metric in the test phase in addition to the performances over downstream tasks.

We then validate the robust performances of our proposed method with two robust baselines from an information-theoretical perspective. Specifically, we utilize GRV~\cite{GRV, LocalGlobal} mentioned in Def.~\ref{lemma-GRV} to indirectly quantify the GCL's robustness. It is worth noting that GRV measures the difference between the mutual information of the node embeddings for the clean graph and the poisoned graph. Here we utilize the InfoNCE loss to approximate the mutual information. We emphasize that we only use the clean graph's information on computing GRV instead of training the robust model and hence prevent the information leakage. We compare \textbf{GCIR} with SPAN and ARIEL in this part and report the GRVs for each robust model in Fig.~\ref{fig-GRV}. A lower GRV indicates better robustness. It is observed that the \textbf{GCIR} achieves the minimum GRV for both CLGA and Mettack, demonstrating that the node embeddings of \textbf{GCIR} are more robust than those of the other models from an information-theoretical perspective. 

\subsection{Time Cost}
\begin{table}[h]
	\centering
	\caption{Training time per epoch (s).}
	\label{tab-time}
	\resizebox{0.95\columnwidth}{!}{%
	\begin{tabular}{cccccc|ccc}
        \toprule[1.pt]
Dataset & BGRL  & DGI   & MVGRL & GRACE & GCA   & ARIEL & SPAN  & \textbf{GCIR}  \\
\hline
Cora      & $0.011$ & $0.005$ & $0.008$ & $0.008$ & $0.009$ & $0.084$ & $0.245$ & $0.045$ \\
CiteSeer  & $0.017$ & $0.006$ & $0.011$ & $0.010$ & $0.009$ & $0.068$ & $0.174$ & $0.041$ \\
Cora-ML & $0.029$ & $0.006$ & $0.013$ & $0.015$ & $0.011$ & $0.111$ & $0.310$ & $0.067$ \\
Photo & $0.031$ & $0.010$ & $0.019$ & $0.038$ & $0.051$ & $1.056$ & $4.347$ & $0.468$ \\
Computers & $0.059$ & $0.023$ & $0.031$ & $0.099$ & $0.090$ & $5.198$ & $8.694$ & $2.293$ \\
WikiCS & $0.043$ & $0.018$ & $0.026$ & $0.067$ & $0.066$ & $4.085$ & $7.394$ & $1.360$ \\
        \bottomrule[1.pt]
        \end{tabular}}
\end{table}

Tab.~\ref{tab-time} presents the time required for each iteration of the training phase of the GCL models. The last three columns show the time cost of using the learnable augmentation views. It is observed that the GCL model with learnable augmentation views requires more computation time than the others. Moreover, our model consumes the least amount of computation time compared to the other two models with learnable augmentation views.  

\subsection{Discussions on Heterophilic Graphs}

\begin{figure}[h]
	\begin{center}
		\subfloat[Chameleon]{\includegraphics[width=0.48\textwidth,height=4.5cm]{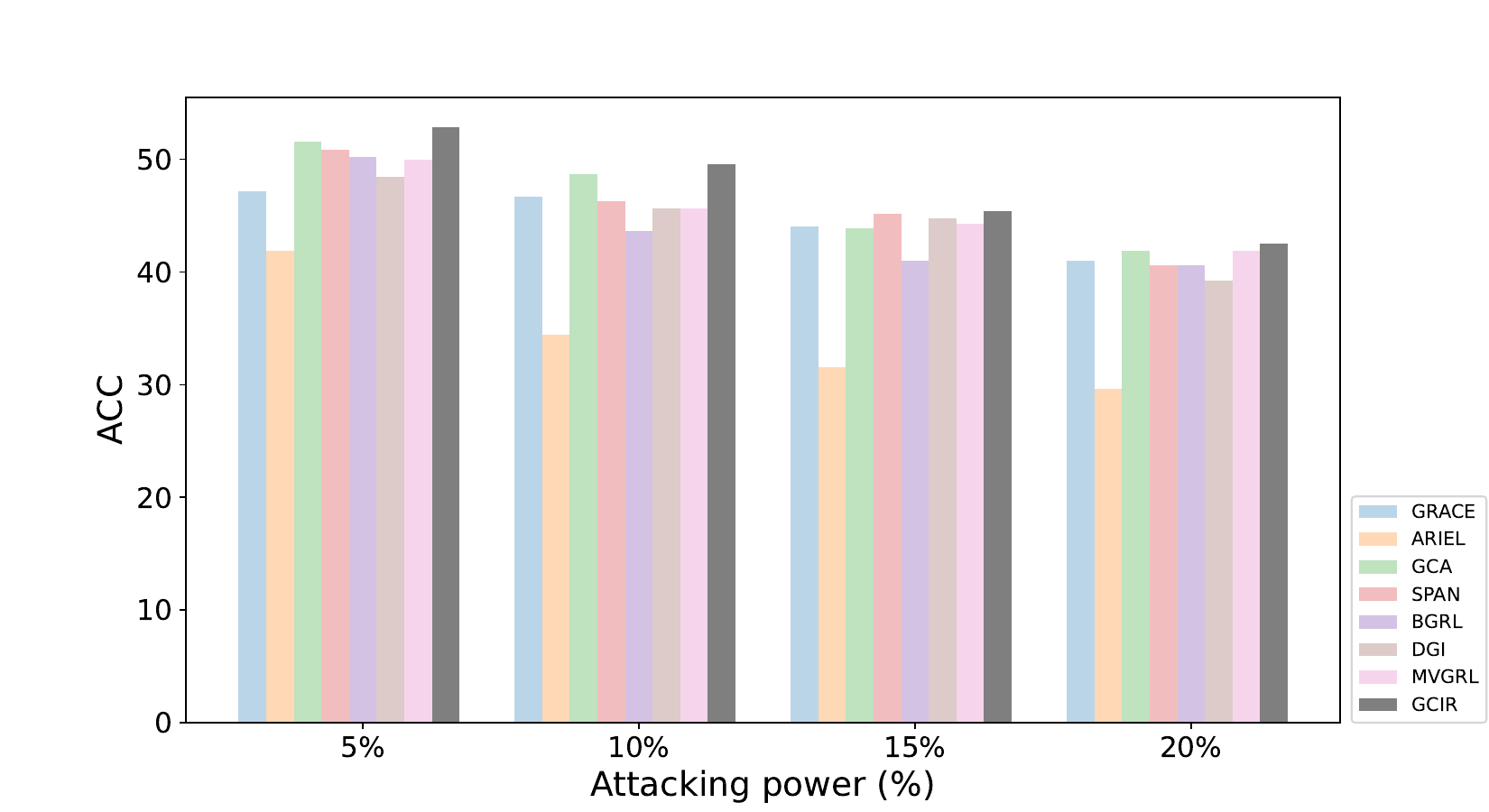}}
        \hfill		
        \subfloat[Squirrel]{\includegraphics[width=0.48\textwidth,height=4.5cm]{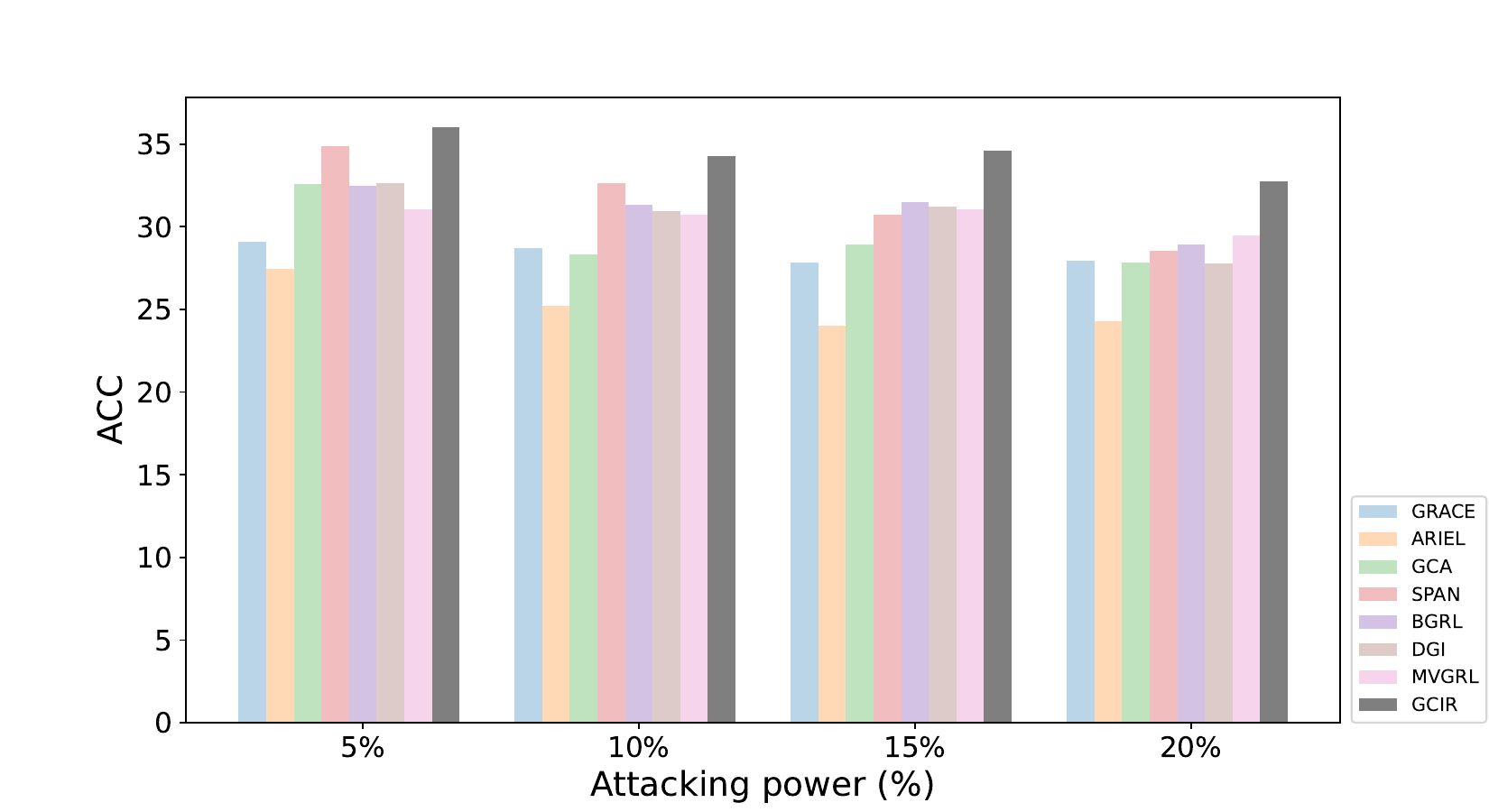}}
	\end{center}
	\caption{Models' Robustness over Heterophilic Graphs.}
	\label{fig-heter}
\end{figure}

In this section, we investigate the adversarial robustness of graph contrastive learning methods on heterophilic graphs. We take two typical heterophilic graphs, i.e., Chameleon and Squirrel~\cite{chameleon} as examplars. Similarly, we deploy Mettack to inject adversarial noises into the clean heterophilic graphs with different attacking powers and evaluate the adversarial robustness of GCLs on the poisoned graphs. The experimental results in Fig.~\ref{fig-heter} demonstrate that our proposed method can still consistently outperform other baselines on heterophilic graphs. It is worth noting that the hyperparameter $\eta$ in Eqn.~\ref{eqn-san-obj-2} is tuned to be zero under this case. It is due to the fact that the feature smoothness is invalid for heterophilic graphs whose node pairs are dissimilar in the majority. Hence, the proposed graph contrastive learning framework solely relies on the information restoration mechanism to achieve robustness. In future work, we will follow this research line and investigate how to further enhance the adversarial robustness of the graph contrastive learning framework over heterophilic graphs while maintaining its robustness over homophilic graphs.

%% file: sections/conclusion.tex
\section{Conclusion}
In this paper, we investigate the adversarial robustness of the GCL against the graph structural attacks from an information-theoretical perspective. The vulnerability analysis offers the clue that apart from the conventional observation that structural attacks tend to diminish the graph homophily, these attacks also degenerate the mutual information estimation between the graph and its representations, which are the cornerstone of the powerful representation learning of GCL. Based on our findings, we propose \textbf{GCIR} with a learnable sanitation view to restore the diminished mutual information and graph homophily after attacks and thus achieve adversarial robustness. The whole framework is trained in an end-to-end learning manner. Besides, we novelly craft an unsupervised tuning strategy to easily determine the best choice of the GCL's vital hyperparameter without avoiding the fully unsupervised setting during the training of GCL. Extensive experiments demonstrate that our method can prominently boost the robustness of the GCL model.   

\section*{Acknowledgement}
This research was partly supported by the Hong Kong RGC Project (No. PolyU25210821).